\documentclass[sigconf]{aamas}

\usepackage{balance} 
\usepackage[utf8]{inputenc} 
\usepackage[T1]{fontenc}    
\usepackage{hyperref}       
\usepackage{url}            
\usepackage{booktabs}       
\usepackage{amsfonts}       
\usepackage{nicefrac}       
\usepackage{microtype}      
\usepackage{xcolor}         
\usepackage{tikz}

\usepackage{amsthm}
\usepackage{algorithm}
\usepackage[noend]{algorithmic}
\usepackage{amsfonts}
\usepackage{amsmath}
\usepackage{arydshln}
\usepackage{graphicx}

\usepackage{adjustbox}
\usepackage[noabbrev]{cleveref}
\usepackage{subcaption}

\usepackage{thmtools}
\usepackage{amsmath, amsthm, amsfonts}
\usepackage{bm}
\usepackage{mathtools}
\usepackage{arydshln}
\usepackage{enumitem} 

\newcommand{\scenario}{\sigma}
\newcommand{\scenarioset}{\Sigma}

\newcommand{\bgpop}{\mathcal{B}}
\newcommand{\states}{\mathcal{S}}
\newcommand{\actions}{\mathcal{A}}
\newcommand{\policies}{\Pi}
\newcommand{\deterministicpolicies}{\policies^\text{D}}
\newcommand{\observations}{\mathcal{X}}
\newcommand{\rewardfunc}{\rho}
\newcommand{\transfunc}{P}
\newcommand{\horizon}{T}
\newcommand{\mdp}{\mu}

\newcommand{\afocal}{\mathbf{a}^f}
\newcommand{\abackground}{\mathbf{a}^b}

\newcommand{\perf}{U_\text{avg}}
\newcommand{\wcu}{U_\text{min}}
\newcommand{\wcr}{R_\text{max}}

\newcommand{\obsfunc}{\mathcal{O}}

\newcommand{\poptest}{\bgpop^\text{test}}
\newcommand{\poptrain}{\bgpop^\text{train}}

\newcommand{\scenariotest}{\scenarioset^\text{test}}
\newcommand{\rmax}{\lVert \rewardfunc \rVert_\infty}

\newcommand*{\defeq}{\coloneqq}

\def\multiset#1#2{\ensuremath{\left(\kern-.3em\left(\genfrac{}{}{0pt}{}{#1}{#2}\right)\kern-.3em\right)}}

\newcommand*\diff{\mathop{}\!\mathrm{d}}

\newenvironment{proofsketch}{%
\renewcommand{\proofname}{Proof Sketch}\proof}{\endproof}
  
\newcommand{\bpi}{\boldsymbol{\pi}}
\newcommand{\minuscule}{\fontsize{4}{6}\selectfont}
\newcommand{\little}{\fontsize{7}{6}\selectfont}

\definecolor{carminepink}{rgb}{0.76, 0.2, 0.26}
\definecolor{celestialblue}{rgb}{0.10, 0.29, 0.62}

\setcopyright{ifaamas}
\acmConference[AAMAS '25]{Proc.\@ of the 24th International Conference
on Autonomous Agents and Multiagent Systems (AAMAS 2025)}{May 19 -- 23, 2025}
{Detroit, Michigan, USA}{A.~El~Fallah~Seghrouchni, Y.~Vorobeychik, S.~Das, A.~Nowe (eds.)}
\copyrightyear{2025}
\acmYear{2025}
\acmDOI{}
\acmPrice{}
\acmISBN{}

\acmSubmissionID{441}

\title[]{A Minimax Approach to Ad Hoc Teamwork}

\author{Victor Villin}
\affiliation{
  \institution{Universit\'{e} de Neuch\^{a}tel}
  \city{Neuch\^{a}tel}
  \country{Switzerland}}
\email{victor.villin@unine.ch}

\author{Thomas Kleine Buening}
\affiliation{
  \institution{The Alan Turing Institute}
  \city{London}
  \country{United Kingdom}}
\email{tbuening@turing.ac.uk}

\author{Christos Dimitrakakis}
\affiliation{
  \institution{Universit\'{e} de Neuch\^{a}tel}
  \city{Neuch\^{a}tel}
  \country{Switzerland}}
\email{christos.dimitrakakis@unine.ch}

\begin{abstract}
We propose a minimax-Bayes approach to Ad Hoc Teamwork (AHT) that optimizes policies against an adversarial prior over partners, explicitly accounting for uncertainty about partners at time of deployment. Unlike existing methods that assume a specific distribution over partners, our approach improves worst-case performance guarantees. Extensive experiments, including evaluations on coordinated cooking tasks from the Melting Pot suite, show our method's superior robustness compared to self-play, fictitious play, and best response learning.  
Our work highlights the importance of selecting an appropriate training distribution over teammates to achieve robustness in AHT.
\end{abstract}

\keywords{Multi-Agent Reinforcement Learning; Ad Hoc Teamwork}

\newcommand{\BibTeX}{\rm B\kern-.05em{\sc i\kern-.025em b}\kern-.08em\TeX}

\begin{document}


\pagestyle{fancy}
\fancyhead{}


\maketitle 


\section{Introduction}



Domain generalisation is often crucial in Reinforcement Learning (RL) and is typically assessed by placing an agent in novel environments \citep{cobbe_quantifying_generalization_reinforcement_2019}. Likewise, in Multi-Agent Reinforcement Learning (MARL), generalisation to new agents can be evaluated by pairing a trained policy with unseen actors \citep{barrett_empirical_evaluation_ad_2011, hu_other_play_zero_2020, leibo_scalable_evaluation_multi_2021, agapiou_melting_pot_2_2023}. While zero-shot domain adaptation is a valuable property \citep{higgins_darla_improving_zero_2017, schafer_task_generalisation_multi_2022}, it is equally important to ensure proper transfer to new behaviours in multi-agent settings, especially in situations where undesired interactions may arise \citep{gleave_adversarial_policies_attacking_2019}. More specifically, Ad Hoc Teamwork (AHT) occurs when an agent, initially unfamiliar with its teammates, must collaborate to achieve a common goal. In a world where autonomous agents are being progressively introduced in such tasks, cooperation with humans is becoming a major concern \citep{stone_ad_hoc_autonomous_2010, ji_ai_alignement_comprehensive_2023}.

Efforts in AHT have primarily focused on learning and inferring behavioural models or teammates types \citep{barrett_empirical_evaluation_ad_2011, albrecht_empirical_study_practical_2015, barrett_making_friends_fly_2017, chen_aateam_achieving_ad_2020, muglich_generalized_beliefs_cooperative_2022}, adapting to behaviour shifts \citep{manish_ad_hoc_teamwork_2019}, and enhancing generalisation by encouraging diversity in partners during training \citep{jaderberg_human_level_performance_2019,hu_other_play_zero_2020, charakorn_investigating_partner_diversification_2020,lupu_trajectory_diversity_zero_2021, strouse_collaboration_with_humans_2021}. However, these methods provide limited guarantees on the worst-case AHT performance.


A multi-agent system can encompass numerous and diverse \emph{scenarios}, each characterised by its actors. For example, autonomous cars operate alongside human drivers and other autonomous vehicles. Similarly, in a surgical setting, a robot may need to cooperate with surgeons who have a wide range of different habits and expertise. In each of these scenarios, we can adopt the perspective that the \emph{focal} actors are controlled by an automated agent, whereas the other actors are viewed as fixed, forming the \emph{background} of the task \citep{leibo_scalable_evaluation_multi_2021, agapiou_melting_pot_2_2023}. These scenarios can be viewed as distinct environments, as each combination of background actors induces different transition dynamics and reward functions. A common practice involves constructing representative scenarios and training a policy on a uniform distribution over them \citep{strouse_collaboration_with_humans_2021,lupu_trajectory_diversity_zero_2021}. However, this only optimises performance for that specific distribution.

Recent studies in zero-shot domain transfer showed that selecting an appropriate prior over training environments is key to learning robust policies \citep{pinto_robust_adversarial_reinforcement_2017, dennis_emergent_complexity_zero_2020, garcin_how_level_sampling_2023, jiang_prioritized_level_replay_2021,buening_minimax_bayes_reinforcement_2023, li_bayes_optimal_robust_2024}. Intuitively, this insight should apply to the AHT setting as well, suggesting that choosing a specific prior over scenarios/partners may improve the robustness of learned policies. Assuming that no information is available about the teammates at test time (and their distribution), we consider the \emph{worst} possible prior over the training set of partners given our policy, an idea adopted from the minimax-Bayes concept \citep{berger_statistical_decision_theory_1985}.
\paragraph{Contributions.} We make the following contributions: 
\begin{enumerate}[topsep=4pt, leftmargin=12pt]
    \item We adapt Minimax-Bayes Reinforcement Learning (MBRL) \citep{buening_minimax_bayes_reinforcement_2023} to the AHT setting, reasoning about uncertainty with respect to partners rather than environments (Section~\ref{seq:robust_aht}).
    \item We examine the advantages of using utility and regret for AHT robustness, and provide solutions to target either metric (Section~\ref{seq:utility_or_regret}).
    \item We study out-of-distribution robustness guarantees (Section~\ref{section:ood}).
    \item We propose a Gradient Descent-Ascent (GDA) \citep{lin_gradient_descent_ascent_2020} based algorithm, in conjunction with policy-gradient methods, and discuss its convergence for softmax policies (Section~\ref{seq:learning}).
    \item We conduct extensive experiments to evaluate our approach. We deploy learned policies on both seen and unseen scenarios for cooperative problems, including a partially observable cooking task from the Melting Pot suite \citep{leibo_scalable_evaluation_multi_2021, agapiou_melting_pot_2_2023}. We compare our approach against Self-Play (SP), Fictitious Play (FP) \citep{brown_iterative_solution_games_1951, heinrich_fictitious_self_play_2015} as well as learning a policy w.r.t.\ a fixed uniform distribution over scenarios \citep{lupu_trajectory_diversity_zero_2021}, which is related to fictitious co-play \citep{strouse_collaboration_with_humans_2021}, as both learn the best response to a population of policies (Section~\ref{seq:experiments}).
    \item Our results confirm the theory and empirically demonstrate that our approach leads to the most robust solutions for both simple and deep RL coordination tasks, even when teammates are adaptive. 
\end{enumerate}
\section{Related Work}

%

\paragraph{Ad Hoc Teamwork.}
In AHT, we are interested in developing agents capable of cooperating with other unfamiliar agents without any form of prior coordination \citep{rovatsos_towards_social_complexity_2002,stone_ad_hoc_autonomous_2010, barrett_empirical_evaluation_ad_2011, barrett_making_friends_fly_2017}. Popular approaches usually involve some form of Population Play (PP), where policies forming a population are learning by interacting with each other \citep{lupu_trajectory_diversity_zero_2021, muglich_equivariant_networks_zero_2022, leibo_scalable_evaluation_multi_2021, agapiou_melting_pot_2_2023}. 
Key strategies for ensuring generalisation to new partners include promoting policy diversity within the training population \citep{charakorn_investigating_partner_diversification_2020} and preventing overfitting to training partners \citep{lanctot_unified_game_theoretic_2017}. Both \citet{lupu_trajectory_diversity_zero_2021} and \citet{strouse_collaboration_with_humans_2021} previously showed that learning a best response to a more diverse population leads to improved generalisation. Additionally, \citet{jaderberg_human_level_performance_2019} showed the effectiveness of PP when diversity is encouraged through evolving pseudo-rewards. However, PP still struggles with producing policies that are robustly collaborative with new partners and sometimes exhibits overfitting \citep{carroll_utility_learning_about_2019,leibo_scalable_evaluation_multi_2021,agapiou_melting_pot_2_2023}.

To further improve AHT, several works suggest inferring the teammates' models/types, maintaining a belief about ad hoc partners based on previous interactions within an episode \citep{barrett_empirical_evaluation_ad_2011, albrecht_empirical_study_practical_2015}. This was shown to help substantially in the partially observable setting \citep{gu_online_ad_hoc_2021, ribeiro_making_friends_dark_2023, dodampegama_knowledge_based_reasoning_2023}. Efforts have also been made to improve the learning and generalisation of such models to new partners \citep{barrett_cooperating_unknown_teammates_2015, barrett_making_friends_fly_2017, muglich_generalized_beliefs_cooperative_2022}.



An alternative approach proposed by \citet{li_robust_multi_agent_2019} involves a robust formulation of deep deterministic policy gradients, assuming worst-case teammates. Unlike our setup, they train a joint policy that remains consistent throughout learning, and design their algorithm specifically for deep deterministic policy gradients, whereas our approach is compatible with any policy-gradient algorithm.

Even though the aforementioned methods attempt to improve cooperative robustness, they always assume specific distributions for the partners. For example, \citet{jaderberg_human_level_performance_2019} used a distribution favoring the matchmaking of policies of similar levels with the intuition that the reward signal is stronger in those cases. However, it does not provide any insights on its effects on AHT robustness. As a result, the actual impact of the training partner distribution on robustness is left under-explored. This holds significant potential, as it can be exploited in conjunction with previously studied mechanisms to substantially enhance AHT robustness.





\paragraph{Zero-shot Domain Transfer.}
AHT can be seen as a form of zero-shot domain transfer. Each possible team composition involving the focal agent can be considered a different environment. In the single-agent setting, \citet{jiang_prioritized_level_replay_2021} demonstrated that adapting the training environment distribution by prioritising environments with higher prediction loss (a measure of the policy's lack of knowledge) leads to improved sample efficiency and generalisation. Building on this idea, \citet{garcin_how_level_sampling_2023} prioritised environments where the mutual information between the learning policy's internal representation and the environment identity was lower, using information theory to achieve similar results. The idea of tampering with the environment distribution was also explored by \citet{pinto_robust_adversarial_reinforcement_2017}, who employed a maximin utility formulation to choose continuous adversarial environment perturbations throughout learning. Instead of utility, \citet{dennis_emergent_complexity_zero_2020} stressed the advantages of using regret by proposing a training environment sampling scheme avoiding entirely unsolvable and uninformative environments. Most relevant to this work,  \citet{buening_minimax_bayes_reinforcement_2023} conducted a study over worst-case priors (for both utility and regret) over training environments, and proved that worst-case distributions are well-suited for domain transfer. \citet{li_bayes_optimal_robust_2024} later reaffirmed those results, learning worst-case distributions within ambiguity sets of subjective priors. Finally, there exist works on domain transfer in the MARL setting \citep{schafer_task_generalisation_multi_2022}, but this differs from our focus on transferring to new partners. This related work is consistently in favor of caring about environment distributions for robustness, providing strong motivation to bring this concern to AHT. 

\section{Problem Formulation}


\subsection{Preliminaries}

An $m$-player Partially Observable Markov Game (POMG) is given by a tuple $\mdp = \langle \states, \observations, \actions, \obsfunc, \transfunc, \rewardfunc , \horizon \rangle$ defined on finite sets of states $\states$,
observations $\observations$ and actions $\actions$. The observation function $\obsfunc: \states \times \{1, \dots, m\} \rightarrow \observations$ provides a state space view for each player. In each state, each player $i$ chooses an action $a_i \in \actions$. Following their joint action $\mathbf{a} = (a_1, \dots, a_m) \in \actions^m$, the state is updated according to the transition function $P: \states \times \actions^m \rightarrow \Delta(\states)$. After a transition, each player receives a reward defined by $\rewardfunc: \states \times \actions^m \times \{1, \dots, m\} \rightarrow \mathbb{R}$. The game ends after $\horizon$ transitions. Permuting player indices does not have any effect on $\mdp$. We denote $\rmax$ the maximum absolute step reward.
    
A policy $\pi: \observations \times \actions \times \observations \times \actions \times \dots \times \observations \rightarrow \Delta(\actions)$ is a probability distribution over a single agent's actions, conditioned on that agent's history of observations and actions. We denote $\policies$ the set of all policies and $\deterministicpolicies \subset \policies$ the set of deterministic policies.

\subsection{Scenarios}
Let a \emph{scenario} $\scenario=(c, \bpi^b)$ be defined by its number of \emph{focal} players $c$, and its \emph{background} players $\bpi^b = (\pi^b_1, \dots, \pi^b_{m-c}) \in \policies^{m-c}$. We say we deploy a policy $\pi^f$ in scenario $\scenario$ if the $c$ focal players are equal to $\pi^f$. Hence, in addition to the $m-c$ many background policies $\bpi^b$, there are $c$ many focal policies $\bpi^f = (\pi^f, \dots, \pi^f)$. We denote $ \afocal \in \actions^c$ and $\abackground \in \actions^{m-c}$ the joint actions of the focal and background players, respectively.
A background population $\bgpop \subset \policies$ is a finite set of policies, to which we assign a set of scenarios:\footnote{The definition of background populations in \eqref{eq:background_population} is largely inspired from the work of \citet{leibo_scalable_evaluation_multi_2021} and \citet{agapiou_melting_pot_2_2023}. This is the most general formulation of the problem and will be used as is throughout this work. Depending on the game $\mdp$, a restricted set may make more sense, such as limiting to $c=1$.}
\begin{equation}
\label{eq:background_population}
\scenarioset(\bgpop) \defeq \{(c, \bpi^b) \mid 1 \leq c \leq m, \bpi^b \in \bgpop^{m-c}\}.
\end{equation}
A scenario $\scenario$ on $\mdp$ can be viewed as its own $c$-player POMG, through the marginalisation of the policies of its background players.\footnote{Each scenario can be seen as a decentralised partially observable Markov decision process \citep{oliehoek_decentralized_pomdps_2012} constrained by the fact that the $c$ players are copies.} We denote $\mdp(\scenario) = \langle \states, \observations, \actions, \obsfunc_\scenario, \transfunc_\scenario, \rewardfunc_\scenario, \gamma ,\horizon \rangle$ the POMG induced by scenario $\scenario$, where $\obsfunc_\scenario: \states \times \{1, \dots, c\} \rightarrow \observations $ is the corresponding observation function, $P_\scenario: \states \times \actions^c \rightarrow \Delta(\states)$ is the transition function given by  
\begin{equation*}
P_{\scenario}(s' \mid s, \afocal) =
\begin{cases}
P(s' \mid s, \afocal), & c = m \\
\sum_{\abackground} \Big(P(s' \mid s, \afocal, \abackground) 
\prod_{i} \bpi^b_i (\abackground_i \mid h_i)\Big), & c < m \\
\end{cases}
\end{equation*}
and $\rho_\scenario: \states \times \actions^c \times \{1, \dots, c\} \rightarrow \mathbb{R}$ is the induced reward function with:
\begin{equation*}
\rho_\scenario(s, \afocal, i) =
\begin{cases}
\rho(s, \afocal, i), & c = m\\
\sum_{\abackground} \Big( \rho(s, \afocal, \abackground, i)
\prod_{j} \bpi^b_{j}(\abackground_j |h_j)\Big), & c < m \\
\end{cases}
\end{equation*}
where $h_j$ is the history of observations and actions of the $j$-th policy and $\abackground_j$ its action in $\abackground$. We denote the scenario that only involves the focal policy, i.e. the universalisation scenario~ \citep{leibo_scalable_evaluation_multi_2021}, by $\scenario^\text{SP} = (m, \emptyset)$. This setup is closely related to N-agent AHT \citep{wang_n_agent_ad_2024}, in that the proportion of focal/background players can vary and is not known in advance.

\subsection{Evaluation}
The expected utility of a policy $\pi$ in scenario $\scenario$ is the mean return of the focal policies given by the expected \emph{focal-per-capita return} \citep{leibo_scalable_evaluation_multi_2021, agapiou_melting_pot_2_2023}:
\begin{equation}
    \label{eq:EU}
    U(\pi, \scenario) \defeq 
    \sum_{t=1}^T  \dfrac{1}{c} \sum_{i=1}^c \mathbb{E}^{\pi}_{\mdp(\scenario)}\left[\rho_\scenario(s_t, \afocal_t, i) \right].
\end{equation}
$U^*(\scenario) \defeq \max_{\pi\in \Pi} U(\pi, \scenario)$ denotes the maximal utility achievable in scenario $\scenario$. This definition for utility represents the need for autonomous agents to always maximise the mean joint rewards of its copies, regardless of the scenario.
We can further define the notion of regret incurred by deploying some policy $\pi$ on scenario $\scenario$, as the gap between the maximal utility and the utility of $\pi$ on $\scenario$:
\begin{equation}
    \label{eq:regret}
    R(\pi,\scenario) \defeq U^*(\scenario) - U(\pi, \scenario).
\end{equation}
To assess a learning method in terms of AHT, we use the evaluation protocol of \citet{leibo_scalable_evaluation_multi_2021}. This has two phases:
\begin{enumerate}[leftmargin=12pt]
    \item \textbf{Training phase}: A background population $\poptest$ is kept hidden. The policy learner has access to the game $\mdp$ with no restrictions, apart from accessing $\poptest$. For example, the learner is free to use a modified instance $\mdp'$ of $\mdp$, where the observation function, $\obsfunc$, may be adjusted to include information about other players, or where the reward function, $\rewardfunc$, could be altered to provide joint rewards instead of individual ones.
    \item \textbf{Testing phase}: The obtained policy is fixed and cannot be trained any further. We compute the performance of the policy on $\mdp$ by taking its average expected utility across a series of unseen test scenarios $\scenariotest \subset \scenarioset(\poptest)$:
  \begin{equation}
        \label{eq:metrics.performance}
        \perf(\pi, \scenarioset) \defeq \dfrac{1}{|\scenarioset|}\sum_{\scenario \in \scenarioset} U(\pi, \scenario),
    \end{equation}
    In addition, we consider two metrics related to robustness, namely worst-case utility and worst-case regret:
    \begin{equation}
        \label{eq:metrics.utiliy}
        \wcu(\pi, \scenarioset) \defeq \min_{\scenario \in \scenarioset} U(\pi, \scenario), \quad\;\;  \wcr(\pi, \scenarioset) \defeq \max_{\scenario \in \scenarioset} R(\pi, \scenario).
    \end{equation}
    Maximising $\wcu$ is typically preferable when falling below a certain utility threshold must be avoided at all costs; for instance minimising casualties in a surgical context. Conversely, minimising $\wcr$ avoids decisions that lead to significantly worse utility than the optimal utility. 
\end{enumerate}
The final objective is to design a learning process outputting a policy that reliably maximises utility or minimises regret on possibly unseen scenarios. 

\subsection{General Assumptions}
\label{sec:assumptions}
To ensure our setting aligns with the AHT literature, we must adhere to three assumptions \citep{mirsky_survey_ad_hoc_2022}: a) the absence of prior coordination. The learner must be capable of cooperating with the team on-the-fly, without relying on previously established collaboration strategies. b) There is no control over teammates, the learner can control its own copies but not other agents in the configuration. c) All agents (focal and background) are assumed to have a partially shared objective. Their reward function may be slightly different, reflecting varying preferences. In this work, we choose to address this last point by assuming a class of possible reward functions for the background players. 

\section{Achieving Robust AHT}
\label{seq:robust_aht}


To learn a policy able to cooperate with new partners, a straightforward idea is to reconstruct scenarios that would be encountered in nature. A roadblock to this approach however is that it requires two main ingredients: a) a diverse pool of partners, and b) a prior distribution over them. The prior, often neglected, is important as it captures our uncertainty about the true partners observed in nature.

In Section~\ref{subsec:constructing_training_scenarios}, we reflect on motivating previous work on diverse behaviour generation, before describing our own adopted approach. Section~\ref{subsec:minimax_bayes_aht} then introduces the Minimax-Bayes idea to AHT, by stating the connections of our setting to Minimax-Bayes Reinforcement Learning (MBRL).

\subsection{Constructing Training Scenarios}
\label{subsec:constructing_training_scenarios}


Before learning any robust policy, we need to construct a diverse set of scenarios. A background population that encompasses a wide range of behaviours is needed in order to reconstruct scenarios existing in nature. Previous work on AHT tackled the issue in various manners, such as using genetic algorithms \citep{muglich_generalized_beliefs_cooperative_2022}, rule-based policies generated with MAP Elites \citep{canaan_generating_adapting_diverse_2023}, SP policies \citep{strouse_collaboration_with_humans_2021}, explicit behavior diversification through regularisation \citep{lupu_trajectory_diversity_zero_2021}, or through evolved pseudo-rewards \citep{jaderberg_human_level_performance_2019}. Based on real-life examples and aiming to thoroughly assess the effects of partner priors, we adopt the following approach:
\begin{itemize}[leftmargin=12pt]
    \item We assume a class of reward functions for background policies:
    \begin{equation*}
        \rho_\text{social+risk} (s, \mathbf{a}, i) = \rho_\text{social}^+ (s, \mathbf{a}, i) - \delta_i \rho_\text{social}^- (s, \mathbf{a}, i),
    \end{equation*}
    with $\rewardfunc_\text{social}$ defined as
    \begin{equation*}
        \rewardfunc_\text{social}(s, \mathbf{a}, i) = \lambda_i \rho(s, \mathbf{a}, i) + (1-\lambda_i) \sum_{j=1}^m \rho(s, \mathbf{a}, j),
    \end{equation*}
    where $\rho^+$ and $\rho^-$ are the positive and negative parts of $\rho$, and $\lambda_i$ and $\delta_i$ denoting levels of prosociality \citep{peysakhovich_prosocial_learning_agents_2017} and risk-aversion, respectively. In other words, each background policy has their own preferences ($\lambda_k, \delta_k$).
    \item Policies are organised into sub-populations $\bgpop = \bigcup_k \bgpop_k$ of varying sizes.
    \item Each sub-populations are separately trained using PP.
\end{itemize}
Given the diverse preferences and varying sizes of the sub-populations, distinct habits and established conventions are more likely to emerge from each group \cite{strouse_collaboration_with_humans_2021}. This choice for constructing scenarios ensures a diverse generation of scenarios, important to ablate the effects of various scenario priors on AHT robustness.
Note that this choice for constructing scenarios remains arbitrary and is not the main focus of our work.

\begin{figure}
    \centering
    \includegraphics[width=0.47\textwidth]{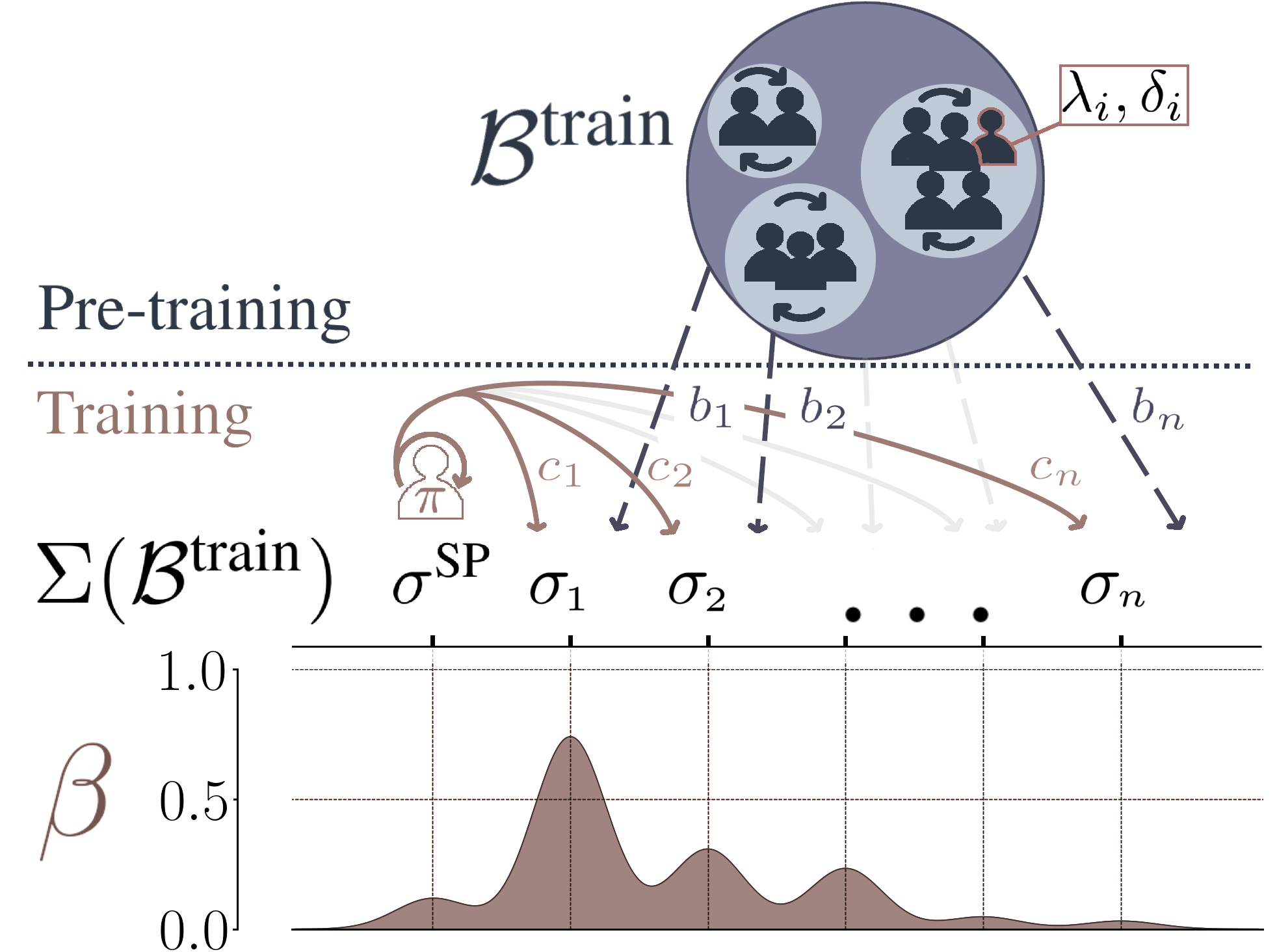}
    \caption{Illustration of the framework used in this paper. Prior to training the focal policy $\pi$, background policies with different preferences ($\lambda_i, \delta_i$) learn by interacting within sub-populations of varying sizes. These sub-populations are then combined to form a background population, $\poptrain$, used as a common ‘training dataset’ for all algorithms.\\ 
    Our primary focus is on the training phase, where the focal policy $\pi$ is trained while the distribution $\beta$ over scenarios is tuned according to the proposed minimax game. These scenarios mix copies of $\pi$ with policies from $\poptrain$, where the self-play scenario $\scenario^\text{SP}$ has the policy interacting only with copies of itself. }
    \label{fig:main_illustration}
\end{figure}

\subsection{Minimax-Bayes AHT}
\label{subsec:minimax_bayes_aht}

In the standard single-agent Bayesian RL setting, the learner selects a subjective belief $\beta$ over candidate Markov Decision Processes (MDPs) $\mathcal{M}$ for the unknown, true environment $\mdp^* \in \mathcal{M}$. The learner's objective is to maximise its expected expected utility with respect to the chosen prior $U(\pi, \beta) = \int_\mathcal{M} U(\pi, \mdp) \diff \beta(\mdp)$, i.e. finding the Bayes-optimal policy. In MBRL, \citet{buening_minimax_bayes_reinforcement_2023} proposed considering the worst possible prior for the agent, without knowledge of the policy that will be chosen. This approach can be interpreted as nature playing the minimising player against the policy learner in a simultaneous-move zero-sum normal-form game. Learning against a worst-case prior intuitively makes the policy more robust, as it prepares for the worst outcomes.

To transfer this idea to our setting, we remark that any finite background population $\bgpop$ provides a finite set of POMGs $\mathcal{M}_\bgpop = \{\mdp(\scenario) | \scenario \in \scenarioset(\bgpop)\}$. The principal difference here is the use of POMGs rather than MDPs. We extend the notion of expected utility with respect to a prior over scenarios, i.e. when $\beta \in \Delta(\scenarioset(\bgpop))$:
\begin{equation}
    U(\pi, \beta) \defeq \mathbb{E}_{\scenario \sim \beta}[U(\pi, \scenario)] = \sum_{\scenario} U(\pi, \scenario) \beta(\scenario).
\end{equation}
This allows us to formulate the following maximin game:\footnote{
If we have a subjective prior, we could learn the distribution within an $\epsilon$-ball around that prior \citep{li_bayes_optimal_robust_2024}. We however consider the full simplex for simplicity.
}
\begin{equation}
    \label{eq:mbmarl.maximin}
    \max_{\pi \in \policies} \min_{\beta \in \Delta(\scenarioset(\bgpop))} U(\pi, \beta).
\end{equation}
Similarly to \citet{buening_minimax_bayes_reinforcement_2023}, we are interested in knowing whether such a game has a solution (i.e., a value), assuming that nature and the agent play simultaneously without knowledge of each other's move. This is relevant in our setting because the policy learner does not know the true distribution of partners available in nature, while nature's distribution does not depend on the policy that will be picked. Fortunately, \eqref{eq:mbmarl.maximin} has a value when $\bgpop$ is finite.
\begin{corollary}[\citet{buening_minimax_bayes_reinforcement_2023}]
For an $m$-player POMG $\mdp$ in a finite state-action space, with a known reward function and a finite horizon, and a background population $\bgpop$, the maximin game \eqref{eq:mbmarl.maximin} has a value:
\begin{equation}
    \label{eq:maximin_value}
    \max_{\pi \in \policies} \min_{\beta \in \Delta(\scenarioset(\bgpop))} U(\pi, \beta) = \min_{\beta \in \Delta(\scenarioset(\bgpop))} \max_{\pi \in \policies} U(\pi, \beta).
\end{equation}
\end{corollary}
\begin{proof}
First, observe that for any stochastic policy $\pi \in \policies$, there exists a distribution over deterministic policies $\phi \in \Delta(\deterministicpolicies)$ such that $\pi(a_t|h_t) = \sum_{d \in \deterministicpolicies} d(a_t|h_t) \phi(d)$. Consequently, we can rewrite the utility as $U(\pi, \beta) = \sum_{d \in \deterministicpolicies} \sum_{\sigma \in \scenarioset(\bgpop)} U(d, \sigma) \phi(d) \beta(\sigma)$. This demonstrates that $U$ is bilinear in $\phi$ and $\beta$, which allows us to apply the minimax theorem, thus proving the result.
\end{proof}
Importantly, prior work that chooses arbitrarily a fixed prior is limited in terms of robustness guarantees: it only ensures maximal utility for their specific prior. In contrast, a policy $\pi^*_U$ solving the maximin utility problem \eqref{eq:mbmarl.maximin} has its utility lower-bounded on $\scenarioset(\bgpop)$:
\begin{equation}
    \label{eq:utility_lower_bound}
    \forall \beta \in \Delta(\scenarioset(\bgpop)), \quad U(\pi^*_U, \beta) \geq U(\pi^*_U, \beta^*_U),
\end{equation}
where $\beta^*_U$ is the worst-case prior for $\pi^*_U$.
Simply put, $\pi^*_U$ performs the worst when the prior is its worst-case $\beta^*_U$, but can only improve when the prior deviates from $\beta^*_U$. Additionally, it is also optimal on the worst-case prior:
\begin{equation}
\label{eq:best_on_worst_case_prior}
    \forall \pi\in\policies, \quad U(\pi^*_U, \beta^*_U) \geq U(\pi, \beta^*_U).
\end{equation}
Note that this differs fundamentally from merely finding the best response to a fixed worst-case prior $\arg\max_{\pi} U(\pi, \beta^*_U)$, which once again, only has a guaranteed optimal utility on $\beta^*_U$.
\begin{corollary}[\citet{buening_minimax_bayes_reinforcement_2023}]
    \label{corollary:min_dirac}
    For any policy $\pi\in\policies$ and background population $\bgpop \subset \policies$, we have 
    \begin{equation}
        \min_{\beta\in\Delta(\scenarioset(\bgpop))} U(\pi, \beta)= \wcu(\pi, \scenarioset(\bgpop)). 
    \end{equation}
\end{corollary}
\begin{proof}
    This follows directly from the results of \citet{buening_minimax_bayes_reinforcement_2023}, using utility in place of regret and recognising that Dirac distributions associated with scenarios in $\scenarioset(\bgpop)$ are always contained in $\Delta(\scenarioset(\bgpop))$.
\end{proof}
\begin{lemma}
    For any background population $\bgpop \subset \policies$ and $\pi^*_U$ the policy solving the maximin utility game~\eqref{eq:mbmarl.maximin}, we have
    \begin{equation}
    \label{eq:optimal_worst_case_utility}
    \wcu(\pi^*_U, \scenarioset(\bgpop)) = \max_{\pi\in\policies} \wcu(\pi, \scenarioset(\bgpop)).
    \end{equation}
\end{lemma}
\begin{proof}
    By Corollary~\ref{corollary:min_dirac}, we can write that $\max_{\pi}\min_{\beta} U(\pi, \beta)=\max_{\pi}\wcu(\pi, \scenarioset(\bgpop))$.
    However, we also have $\max_{\pi}\min_{\beta} U(\pi, \beta) = \min_{\beta} U(\pi^*_U, \beta) = \wcu(\pi^*_U,\scenarioset(\bgpop))$.
\end{proof}
Thus, a policy solving the maximin utility game~\eqref{eq:mbmarl.maximin} is guaranteed to have an optimal worst-case utility on its training set.


\section{Utility or Regret?}
\label{seq:utility_or_regret}

Optimising for the worst-case utility \eqref{eq:mbmarl.maximin} might be problematic. Nature could resort to only picking scenarios where the focal players achieve the worst possible score. Then, the distribution trivially minimises utility for any chosen policy, preventing the latter to learn anything.
\citet{buening_minimax_bayes_reinforcement_2023} addresses this issue by instead considering the regret of a policy. The difference is that ‘impossible’ scenarios will always yield zero regret for any policy, thus becoming irrelevant for a regret-maximising nature. Letting $L(\pi, \beta) \defeq \sum_\sigma R(\pi, \mdp) \beta(\sigma)$ be the Bayesian regret with respect to a prior $\beta$, we now formulate the following minimax regret game:
\begin{equation}
    \label{eq:mbmarl.minimax}
    \min_{\pi \in \policies} \max_{\beta \in \Delta(\scenarioset(\bgpop))} L(\pi, \beta).
\end{equation}
One can also prove that this above game has a value. Moreover, a solution  ($\pi^*_R, \beta^*_R$) to \eqref{eq:mbmarl.minimax} exhibits properties analogous to those in equations~\eqref{eq:utility_lower_bound}, \eqref{eq:best_on_worst_case_prior} and \eqref{eq:optimal_worst_case_utility}, but in terms of regret. $\pi^*_R$ has its Bayesian regret upper-bounded by $L(\pi^*_R, \beta^*_R)$ on $\scenarioset(\bgpop)$. It is also optimal under the worst-case prior $\beta^*_R$ and achieves optimal worst-case regret $\wcr$ on $\scenarioset(\bgpop)$.

Should utility or regret be used as an objective? Exploiting regret ensures that scenarios on which you can improve the most are sampled more often. It also ensures that degenerate scenarios get discarded as their regret is always zero. However, it demands the calculation of best responses for each scenario, which becomes taxing as the number of scenarios or problem complexity grows.
To reduce the computational burden, we can approximate those best responses, or subsample the set of scenarios.
An alternative way is to make use of the utility notion under some additional conditions.

\begin{definition}[Non-degenerative population]
        A background population of policies $\bgpop \subset \policies$ is non-degenerative if and only if for any scenario $\scenario\in \scenarioset(\bgpop)$, there exists two distinct policies $\pi_1$ and  $\pi_2 \in \policies, \pi_1 \neq \pi_2$ such that $U(\pi_1, \scenario) \neq U(\pi_2, \scenario)$.
    \end{definition}
\begin{lemma}
    If a background population $\bgpop\subset\policies$ is non-degenerative, then
    for any scenario $\scenario \in \scenarioset(\bgpop)$, there exists a policy $\pi \in \policies$ such that $R(\pi, \scenario) > 0$. 
    \begin{proof}
        $\bgpop$ is non-degenerative, for any scenario $\scenario \in \scenarioset(\bgpop)$ there must exist two policies $\pi_1$ and $\pi_2$ such that $U(\pi_1, \scenario) > U(\pi_2, \scenario)$. We have by definition $U^*(\scenario)\geq U(\pi_1, \scenario)$, hence $R(\pi_2, \scenario) > 0$.
    \end{proof}
\end{lemma}
Making the assumption that a background population is non-degenerative is in general realistic for cooperative tasks. This translates into only considering reasonable behaviors for the background population, or tasks where teammates cannot completely cancel out the actions of the focal players. Under the assumption of a non-degenerative background population, no distribution can deadlock the policy learner into stale scenarios. Hence, the utility-minimising opponent in Equation~\ref{eq:mbmarl.maximin} can no longer trivially minimise utility.
For the remainder of the paper, background populations are assumed to be non-degenerative. 





\section{Out-Of-Distribution Robustness}
\label{section:ood}

As already stated in Section~\ref{subsec:constructing_training_scenarios}, having a diverse set of scenarios that adequately represents the true set of scenarios is crucial. However, since it is often impractical to perfectly replicate the true set, the prior used during training may not have the same support as the true distribution observed in nature. In such cases, the guarantees outlined in Section~\ref{subsec:minimax_bayes_aht} no longer hold on the true distribution. In order to state further robustness guarantees, an option is to assume that scenarios in the true scenario set are close to the training scenarios. To quantify the closeness between scenarios, we first define the distance between two policy vectors as their maximum total variation across all states:
\begin{equation}
    \label{eq:policy_vector_distance}
    d(\bpi, \bpi') = \max_{s\in \states }\sum_i \sum_a \big| \bpi_i(a|s) - \bpi'_i(a|s) \big|.
\end{equation}
We define the scenario distance as the minimum distance between policy vectors across permutations of the background policies:
\begin{equation}
    \label{eq:scenario_distance}
    d(\scenario, \scenario') = \min_{\bpi, \bpi' \in \text{Perm}(\bpi^b) \times \text{Perm}(\bpi^{b'})} d(\bpi, \bpi'),
\end{equation}
This metric measures the similarity between the background policies of two scenarios. Scenarios can only be compared if they have the same number of focal players (e.g., $\scenario=(c,\bpi^b)$ and $\scenario'=(c,\bpi^{b'})$).
\begin{definition}[$\epsilon$-net of a scenario set] 
    A finite set of scenarios $\scenarioset$ is called an $\epsilon$-net of a scenario set $S$ if and only if, for every scenario $\scenario\in S$, there exists a scenario $\scenario'\in\scenarioset$ such that $d(\scenario,\scenario')<\epsilon$.
\end{definition}

\begin{lemma}
    \label{lemma:scenario_equivalence}
    Let $\scenarioset$ be an $\epsilon$-net for a scenario set $S$. For any policy $\pi \in \policies$ and scenario $\scenario\in S$, there is a scenario $\scenario' \in \scenarioset$ that verifies:
    \begin{equation}
        \label{eq:performance.guarantee}
        \big| U(\pi, \scenario) - U(\pi, \scenario') \big| < \dfrac{\epsilon T^2\rmax}{2}.
    \end{equation}
\end{lemma}
\begin{proofsketch}
    The result is obtained by using the fact that for any pairs of $\epsilon$-close scenarios $\scenario, \scenario'$ and any $s, \afocal, i$, we have $\sum_{s'}|P_\scenario(s'| s, \afocal) - P_{\scenario'}(s'| s, \afocal) | < \epsilon$ and $|\rewardfunc_\scenario(s, \afocal, i) - \rewardfunc_{\scenario'}(s, \afocal, i)|\\< \epsilon\rmax$. The proof is concluded by showing by induction that for all $t$ and $s$ , $|U_t(\pi, \scenario, s)-U_t(\pi, \scenario', s)| < \frac{1}{2} \epsilon(T-t+1)(T-t)\rmax$.
\end{proofsketch}

\begin{lemma}
    \label{lemma:scenario_equivalence_regret}
    Let $\scenarioset$ be an $\epsilon$-net for some scenario set $S$. For any policy $\pi\in\policies$ and scenario $\scenario\in S$, there is a scenario $\scenario' \in \scenarioset$ such that
    \begin{equation}
        \label{eq:performance.regret}
        \big| R(\pi, \scenario) - R(\pi, \scenario') \big| < \epsilon T^2\rmax.
    \end{equation}
\end{lemma}
\begin{proofsketch}
    The result is obtained by both using the identity $|U^*(\scenario) - U^*(\scenario')| \leq \max_\pi |U(\pi, \scenario) - U(\pi, \scenario')|$
    and noticing that for any policy $\pi$,  $|R(\pi, \scenario)-R(\pi, \scenario')|\leq |U^*(\scenario)- U^*(\scenario')| + |U(\pi, \scenario) - U(\pi, \scenario')|$.
\end{proofsketch}

\begin{lemma}
    \label{lemma:wcu.guarantees}
    Let $\scenarioset$ be an $\epsilon$-net for some scenario set $S$, and $\pi^*_U$the optimal policy for the maximin utility problem \eqref{eq:mbmarl.maximin} on $\scenarioset$, then
    \begin{equation}
        \wcu(\pi^*_U, S) > \max_{\pi \in \policies} \left( \wcu(\pi, \scenarioset) - \dfrac{\epsilon T^2 \rmax}{2} \right). 
    \end{equation}
\end{lemma}
\begin{proofsketch}
    We denote $\scenario_\text{wc}(\scenarioset)$ and $\scenario_\text{wc}(S)$ the worst-case scenarios for $\pi^*_U$ on $\scenarioset$ and $S$, and reason on the distance between $\scenario_\text{wc}(\scenarioset)$ and $\scenario_\text{wc}(S)$. If $d(\scenario_\text{wc}(\scenarioset),\scenario_\text{wc}(S)') < \epsilon$, then Lemma~\ref{lemma:scenario_equivalence} applies.  Otherwise, since $\scenarioset$ is an $\epsilon$-net, we can find another scenario $\scenario_\epsilon \in \scenarioset$ that is $\epsilon$-close to $\scenario_\text{wc}(S)$ and use the fact that the utility of $\pi^*_U$ is by definition higher on $\scenario_\epsilon$ than on $\scenario_\text{wc}(\scenarioset)$.
\end{proofsketch}

\begin{lemma}
    \label{lemma:wcr.guarantees}
    Let $\scenarioset$ be an $\epsilon$-net for some scenario set $S$, and $\pi^*_R$ the optimal policy for the minimax regret problem \eqref{eq:mbmarl.minimax} on $\scenarioset$, then
    \begin{equation}
        \wcr(\pi^*_R, S) < \min_{\pi \in \policies} \big( \wcr(\pi, \scenarioset) + \epsilon T^2 \rmax \big).
    \end{equation}
\end{lemma}
\begin{proofsketch}
        We prove, analogically to Lemma~\ref{lemma:wcu.guarantees}, the result using Lemma~\ref{lemma:scenario_equivalence_regret} in place of Lemma~\ref{lemma:scenario_equivalence}.
\end{proofsketch}

Lemmas~\ref{lemma:wcu.guarantees} and \ref{lemma:wcr.guarantees} provide worst-case guarantees on arbitrary sets of scenarios, for policies solving the minimax problems. This also means that we can have those guarantees on non-finite sets of scenarios. Importantly, as long as we have an $\epsilon$-net of training scenarios for the true set, the policy solving the maximin utility (or minimax regret) problem has a strong worst-case utility (or regret) guarantee. In contrast, it is impossible to guarantee \emph{anything} additional about the average utility $\perf$ on the true set, as the latter could very well include scenarios that are all $\epsilon$-close to the worst-case scenarios of the training set. For this reason, the average utility on the true set can be as low as the worst-case utility.

\begin{algorithm}[t]

\caption{Background-Focal GDA}
\label{algorithm:full_information_gda}

\begin{algorithmic}[1]
\STATE \textbf{Input\,} 
Background policies $\mathcal{B}$, 
and
learning rates ($\eta_\pi$, $\eta_\beta$).
\STATE Simplex projector $\mathcal{P}$
\STATE Initialise the main policy parameters $\theta_0$ randomly
\STATE Initialise the distribution as uniform $\beta_0 = \mathcal{U}(\Sigma(\mathcal{B}))$ 

\FOR{$t=0, \dots, N-1$}

    \STATE Compute $U(\pi_{\theta_t}, \sigma)$ for all $\sigma \in \Sigma(\mathcal{B})$
    \STATE Compute $U(\pi_{\theta_t}, \beta_t) = \sum_\sigma U(\pi_{\theta_t}, \sigma) \beta_t(\sigma)$
    \STATE Compute $R(\pi_{\theta_t}, \sigma) = U^*(\sigma) - U(\pi_{\theta_t}, \sigma)$ for all $\sigma \in \Sigma(\mathcal{B})$

    \STATE Obtain $L(\pi_{\theta_t}, \beta_t) = \sum_\sigma R(\pi_{\theta_t}, \sigma)\beta_t(\sigma)$

    \IF{solving maximin utility \eqref{eq:mbmarl.maximin}}
        \STATE Update distribution $\beta_{t+1} = \mathcal{P}(\beta_t -\eta_{\beta}\nabla_{\beta} U(\pi_{\theta_t}, \beta_t))$
    \ENDIF
    \IF{solving minimax regret \eqref{eq:mbmarl.minimax}}
        \STATE Update distribution $\beta_{t+1} = \mathcal{P}(\beta_t + \eta_{\beta}\nabla_{\beta} L(\pi_{\theta_t}, \beta_t))$
    \ENDIF
    \STATE Update policy parameters $\theta_{t+1} = \theta_t + \eta_{\theta}\nabla_{\theta} U(\pi_{\theta_t}, \beta_t)$
\ENDFOR

\RETURN $\theta^*, \beta^*$ uniformly sampled from $\{(\theta_1, \beta_1), \dots, (\theta_N, \beta_N)\}$
\end{algorithmic}
\end{algorithm}
\section{Computing Solutions}
\label{seq:learning}

We now desire to calculate the solution pairs for both the maximin utility \eqref{eq:mbmarl.maximin} and minimax regret \eqref{eq:mbmarl.minimax} games. \citet{buening_minimax_bayes_reinforcement_2023} theoretically proved that GDA has convergence guarantees when the game is played between a policy learned with softmax parameterisation and nature learning its distribution over a finite set of MDPs. These results apply if all scenarios induce single-agent POMGs, as partial observability does not interfere with proving the required properties. However, when the focal policy is deployed in a scenario with $c>1$ copies, the game is no longer single-agent. To approximate the reduction of these multi-agent POMGs to single-agent POMGs during training, we propose using delayed versions $\pi_{t-d}$ of the focal policy $\pi_{t}$ for the $c-1$ remaining copies. This common practice smooths the behavior of the copies and favours proper convergence by treating the copies as fixed policies.

Algorithm~\ref{algorithm:full_information_gda}, a GDA algorithm adapted to our setting, can be employed to learn solutions for both the maximin utility and minimax regret problems. Furthermore, in case the POMG is not known, one can straightforwardly adapt the algorithm to a stochastic version, by resorting to sampling scenarios and performing policy rollouts to estimate utility, regret, and gradients.

\section{Experiments}
\label{seq:experiments}
The aim of our experiments is to highlight the importance of partner distribution in the learning process. To achieve this, we evaluate our proposed strategies, Maximin Utility (MU) and Minimax Regret (MR), on two distinct problems. First, we consider the fully known and observable Iterated Prisoner's Dilemma to validate the theoretical results. Following this, we test our approaches on a deep-RL task, the Collaborative Cooking (Overcooked) game \citep{carroll_utility_learning_about_2019, strouse_collaboration_with_humans_2021, leibo_scalable_evaluation_multi_2021,agapiou_melting_pot_2_2023}.
We remind that we want to optimise policies with respect to the focal-per-capita return \eqref{eq:EU} rather than individual returns. Naturally, across all of our experiments, we reward focal policies with average focal step rewards. For each scenario $\scenario$, this is defined as $
    \rho^\text{train}_\sigma(s, \afocal, i) = \frac{1}{c} \sum_{j=1}^c \rho_\sigma(s, \afocal, j)
$.
Throughout our experiments, we benchmark against three distribution management strategies: 
\begin{itemize}[leftmargin=12pt]
    \item Population Best Response (PBR): uniform distribution over scenarios,
    \item Self-Play (SP): playing solely with copies of the focal policy,
    \item Fictitious Play (FP): sampling partners uniformly from the history of focal policies $\pi_0, \dots, \pi_t$.
\end{itemize}



\subsection{Iterated Prisoner's Dilemma}
\label{subsec:prisoners}
In these experiments, all computations can be exact. This includes the gradient calculation for the prior, as well as for the agent's policies. We focus on the Iterated Prisoner's Dilemma, where two players play the matrix game (Table~\ref{tab:prisoners_matrix}) repeatedly for $\horizon=3$ rounds. 
\paragraph{Experimental Setup.}
In the Iterated Prisoner's Dilemma, players receive and observe rewards based on their chosen actions, specified by the payoff matrix. The game has one state, and the outcomes observed are enough to determine the joint actions, making it fully observable. We learn softmax, fully adaptive policies, where actions depend on the entire history of observations and actions.
Unlike other experiments, we do not use the approach described in Section~\ref{subsec:constructing_training_scenarios}. Instead, the learner interacts with a background population $\poptrain$ composed of nine ad-hoc policies, such as pure cooperation/defection, tit-for-tat, cooperate-until-defected, and fully random.
For AHT assessment, we adequately construct a test background population $\poptest$, ensuring that its scenarios are $\epsilon$-close to those in the training phase. Specifically, we uniformly sample $512$ stochastic policies that are $\epsilon$-close to the policies in the training background population, setting $\epsilon = 0.5$.

\begin{table}[t]
    \centering
    \caption{Payoff matrix of the Prisoner's Dilemma.}
    \label{tab:prisoners_matrix}
    \begin{tabular}{c|c:c}
         & Cooperate  & Defect \\
         \hline
        Cooperate & $(4, 4)$ & $(0, 5)$\\
        \hdashline
        Defect  & $(5, 0)$ & $(1, 1)$\\
    \end{tabular}
\end{table}
\begin{table}[t]
    \small
    \setlength{\tabcolsep}{3pt}
    \centering
    \caption{
    Scores on the Iterated Prisoner's Dilemma. A higher value is desired for average utility ($\perf$) and worst-case utility ($\wcu$), while a lower worst-case regret ($\wcr$) is better.}
    \label{tab:prisoner.train}
    \begin{tabular}{l||ccc||ccc||} 
    & \multicolumn{3}{c||}{$\scenarioset(\poptrain)$} & \multicolumn{3}{c||}{$\scenarioset(\poptest)$}\\
    \cmidrule(lr){2-4}\cmidrule(lr){5-7}
               & $\perf$ & $\wcu$ & $\wcr$ & $\perf$ & $\wcu$ & $\wcr$ \\ \midrule
        Maximin Utility (MU) & 7.69 & \textbf{3.00} & 9.00 & \textbf{8.34} & \textbf{3.00} & 9.00\\
        Minimax Regret (MR)  & 8.23 & 2.26 & \textbf{3.79} & 7.96 & 2.78 & \textbf{4.35}\\ \hdashline
        Population Best Response (PBR) & \textbf{8.54} & 2.00 & 4.97 & 8.07 & 2.48 & 5.63\\
        Fictitious Play (FP)& 7.06 & 0.14 & 10.56 & 6.33 & 0.56 & 9.96\\
        Self-Play  (SP)& 7.25 & 0.47 & 10.19 & 6.49 & 0.86 & 9.65\\
        Random & 7.40 & 1.50 & 5.50 & 7.47 & 2.18 & 5.20\\
    \end{tabular}
\end{table}

\paragraph{Results.} 
Table~\ref{tab:prisoner.train} summarises the performance on both training and test sets. As expected, on the training set, PBR performs the best under the uniform prior ($\perf$), MU has the highest worst-case utility ($\wcu$), and MR exhibits the lowest worst-case regret ($\wcr$). On the test set however, MU outperforms PBR in terms of average utility and maintains the highest worst-case utility, while MR continues to excel in minimising worst-case regret. These results indicate that best responses to populations does not ensure
the best robustness to new partners. SP and FP agents as well, overfit to their established conventions, leading to poor transferability across training and test policies.
Figure~\ref{fig:prisoners.binary_trees} visualises the learned policies under different distribution regimes, highlighting their disparity. For example, the population best response is a strategy close to "cooperate-until-defected", while MU's policy heavily favors defection. 
Crucially, this figure points to a potential improvement for future work: during optimisation, the worst-case distribution can force policies onto a narrow subset of scenarios, leaving others unexplored. Due to the MU regime, the policy is forced to face a pure defecting opponent, its worst-case scenario, entirely cutting its exposure to cooperative strategies. As seen in Figure~\ref{subfig:prisoners.mu}, the policy does not know what to do if the opponent chooses to collaborate. This suggests that it is possible to have a policy with equal worst-case utility but improved worst-case regret and overall performance.

\input{data/prisoners/binary_trees}

\subsection{Robust AHT on Collaborative Cooking}
\label{subsec:cooking}
For this section, we tackle the Collaborative Cooking game \citep{agapiou_melting_pot_2_2023}, where two players act as chefs in a gridworld kitchen, working together to deliver as many tomato soup dishes as possible within a set time. Each have to collect tomatoes, cook them, prepare dishes, and deliver the soup. Successful deliveries reward both players equally.\footnote{When training background policies specifically, we only assign a reward of 20 points to the player who delivers the dish, and 1 point to players who contribute by placing a tomato into the pot. This incentivises diverse behaviour generation.} Players must navigate the kitchen, interact with objects in the right order, and coordinate with each other. Each player has a local, partial RGB view of the environment. All of our policies in this section are using deep recurrent (LSTM) neural networks.


\begin{figure}
    \centering
    \begin{subfigure}[t]{0.48\linewidth}
    \includegraphics[width=\linewidth]{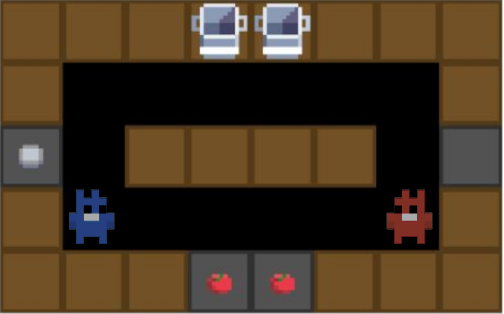}
    \caption{Circuit.}
    \end{subfigure}
    \begin{subfigure}[t]{0.48\linewidth}
    \includegraphics[width=\linewidth]{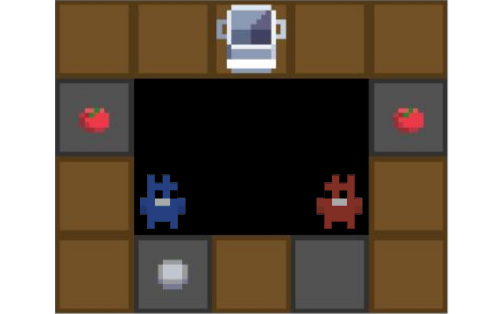}
    \caption{Cramped.}
    \end{subfigure}
    \vspace{-0.1cm}
    \caption{We consider two kitchen layouts in the Collaborative Cooking environment. The players must collaborate to deliver tomato soups without blocking each other. Both players collect a reward of 20 each time a soup is delivered.}
    \label{fig:layouts}
\end{figure}
\begin{table*}[t]
    \small
    \setlength{\tabcolsep}{3pt}
    \centering
    \caption{Scores on the Collaborative Cooking environment training and test sets. The standard error is taken over three random seeds. The scores are aggregated over the two kitchen layouts. Lower worst-case regret $\wcr$ is better. Scores in italic are reported from \citet{agapiou_melting_pot_2_2023}, which were not obtained on the exact same setting.}
    \label{tab:cooking.train}
    \begin{tabular}{l||ccc||ccc||ccc||}
    & \multicolumn{3}{c||}{$\scenarioset(\poptrain)$} & \multicolumn{3}{|c||}{$\scenarioset(\poptest)$} & \multicolumn{3}{c||}{Melting pot scenarios} \\
    \cmidrule(lr){2-4}\cmidrule(lr){5-7}\cmidrule(lr){8-10}
               & $\perf$ & $\wcu$ & $\wcr$ & $\perf$ & $\wcu$ & $\wcr$ & $\perf$ & $\wcu$ & $\wcr$ \\ \midrule
        Maximin Utility (MU) & \textbf{266.9 {\footnotesize $\pm$ 4.3}} & \textbf{225.3 {\footnotesize$\pm$ 11.5}} & 266.0 {\footnotesize$\pm$ 7.9} & \textbf{195.7 {\footnotesize$\pm$ 6.2}} & \textbf{66.0 {\footnotesize$\pm$ 6.8}} & 266.4 {\footnotesize$\pm$ 10.1} & \textbf{273.8 {\footnotesize$\pm$ 4.9}} & \textbf{224.9 {\footnotesize$\pm$ 7.1}} & \textbf{118.0 {\footnotesize$\pm$ 7.1}} \\
        Minimax Regret (MR) & 232.0 {\footnotesize$\pm$ 18.6} & 144.3 {\footnotesize$\pm$ 14.4} & \textbf{230.7 {\footnotesize$\pm$ 28.2}} & 172.2 {\footnotesize$\pm$ 15.4} & 65.1 {\footnotesize$\pm$ 16.0} & \textbf{248.2 {\footnotesize$\pm$ 28.4}} & 206.8 {\footnotesize$\pm$ 12.6} & 148.7 {\footnotesize$\pm$ 9.1} & 187.1 {\footnotesize$\pm$ 13.0} \\ \hdashline
        Population Best Response (PBR) & 209.7 {\footnotesize$\pm$ 23.9} & 96.8 {\footnotesize$\pm$ 13.4} & 357.6 {\footnotesize$\pm$ 16.1} & 151.4 {\footnotesize$\pm$ 19.5} & 33.6 {\footnotesize$\pm$ 5.9} & 327.1 {\footnotesize$\pm$ 14.0} & 171.8{\footnotesize $\pm$ 21.1} & 106.3 {\footnotesize$\pm$ 9.3} & 228.1 {\footnotesize$\pm$ 5.8} \\
        Fictitious Play (FP) & 129.9 {\footnotesize$\pm$ 13.9} & 0.2 {\footnotesize$\pm$ 0.1} & 483.5 {\footnotesize$\pm$ 16.1} & 152.2 {\footnotesize$\pm$ 16.8} & 16.7 {\footnotesize$\pm$ 11.7} & 369.7 {\footnotesize$\pm$ 19.7} & 121.5 {\footnotesize$\pm$ 15.7} & 40.7 {\footnotesize$\pm$ 16.3} & 294.8 {\footnotesize$\pm$ 10.7} \\
        Self-Play (SP) & 124.8 {\footnotesize$\pm$ 26.4} & 15.7 {\footnotesize$\pm$ 10.5} & 460.8 {\footnotesize$\pm$ 21.7} & 117.4 {\footnotesize$\pm$ 12.5} & 6.7 {\footnotesize$\pm$ 3.5} & 367.5 {\footnotesize$\pm$ 11.6} & 101.2 {\footnotesize$\pm$ 
17.8} & 29.0 {\footnotesize$\pm$ 17.4} & 293.4 {\footnotesize$\pm$ 14.5} \\
        Random & 42.8 {\footnotesize$\pm$ 0.0} & 0.0 {\footnotesize$\pm$ 0.0} & 505.4 {\footnotesize$\pm$ 0.0} & 32.2 {\footnotesize$\pm$ 0.0} & 0.0 {\footnotesize$\pm$ 0.0} & 445.0 {\footnotesize$\pm$ 0.0} & 60.6 {\footnotesize$\pm$ 0.0} & 0.0 {\footnotesize$\pm$ 0.0} & 307.3 {\footnotesize$\pm$ 0.0} \\ 
        PP-ACB & n/a & n/a & n/a & n/a & n/a & n/a & \textit{82.4 {\footnotesize$\pm$ 0.0}} & \textit{0.0 {\footnotesize$\pm$ 0.0}} & \textit{307.3 {\footnotesize$\pm$ 0.0}} \\
        PP-OPRE & n/a & n/a & n/a & n/a & n/a & n/a & \textit{102.3 {\footnotesize$\pm$ 0.0}} & \textit{14.6 {\footnotesize$\pm$ 0.0}} & \textit{292.7 {\footnotesize$\pm$ 0.0}} \\
        PP-VMPO & n/a & n/a & n/a & n/a & n/a & n/a & \textit{78.6 {\footnotesize$\pm$ 0.0}} & \textit{36.1 {\footnotesize$\pm$ 0.0}} & \textit{306.7 {\footnotesize$\pm$ 0.0}} \\
    \end{tabular}
\end{table*}

\begin{figure}
    \centering
    \begin{subfigure}[t]{\linewidth}
    \includegraphics[width=0.498\linewidth]{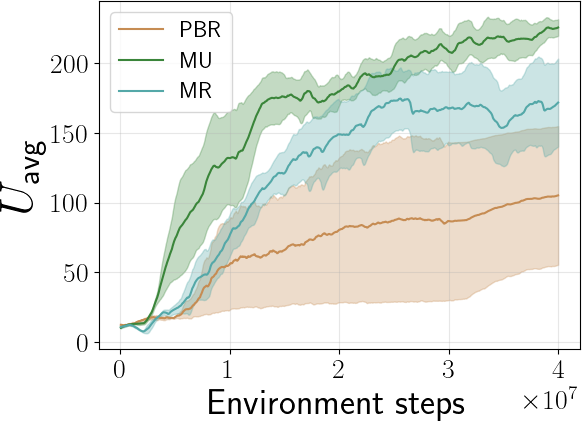}
    \includegraphics[width=0.498\linewidth]{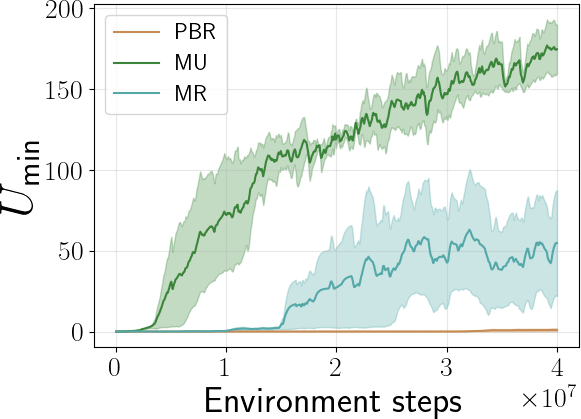}
    \caption{Circuit. \label{fig:cooking:learning_curves.circuit}}
    \end{subfigure}
    \begin{subfigure}[t]{\linewidth}
    \includegraphics[width=0.498\linewidth]{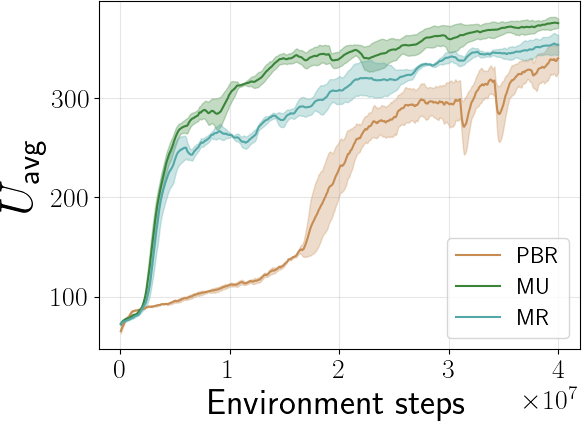}
    \includegraphics[width=0.498\linewidth]{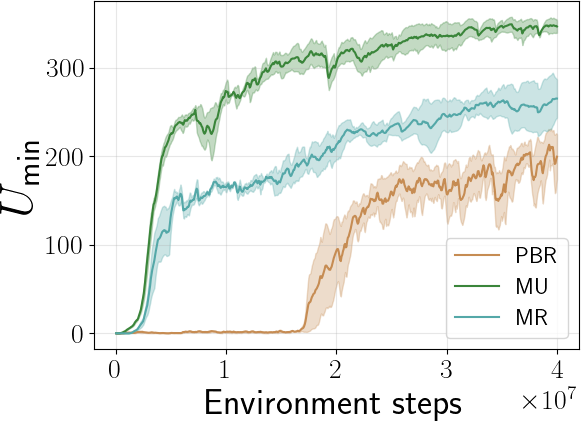}
    \caption{Cramped.}
    \end{subfigure}
    \caption{Learning curves of the average and worst-case utility metrics over the training set (averaged over 3 runs). The standard error is shown in shaded colour.}
    \label{fig:cooking.learning_curves}
\end{figure}

\paragraph{Experimental Setup.} Two separate background populations, $\poptrain$ and $\poptest$, are generated according to Section~\ref{subsec:constructing_training_scenarios}. Both populations are trained with an identical setup, differing only in their seed. Each is partitioned into four sub-populations of sizes 2, 3, and 5, totaling 10 policies. Prosociality and risk-aversion for all background policies are sampled uniformly in $[-0.2, 1.2]$ and $[0.1,2]$ respectively.\footnote{Those intervals were chosen empirically to ensure diverse and cooperative policies.}
For a fair comparison and to focus on scenario distribution learning, we assume that $\poptrain$ is readily available to all approaches, which can be exploited for a maximum of $4 \times 10^7$ environment steps to learn a policy with PPO \citep{schulman_proximal_policy_optimization_2017}. We run each training procedure on three different seeds (solely $\poptrain$ and $\poptest$ remain consistent across runs). We evaluate the approaches on two different kitchen layouts: Circuit and Cramped \citep{agapiou_melting_pot_2_2023}, which are visualised in Figure~\ref{fig:layouts}. 
On top of our own test scenarios, we assess the learned policies on the Melting Pot benchmark scenarios, comparing them against the baselines reported in the original paper \citep{agapiou_melting_pot_2_2023}: an Actor-Critic Baseline (ACB), V-MPO \citep{song_vmpo_on_policy_2019}, and OPRE \citep{vezhnevets_options_responses_grounding_2020}.\footnote{Since their baseline policies are not publicly available, we were unable to evaluate them on our scenarios, which explains the missing values in Table~\ref{tab:cooking.train}.} These baselines were trained for $10^9$ steps without access to our background policies. In each scenario (including the Melting Pot scenarios), regret was computed by estimating best responses with PPO for $10^7$ steps. Lastly for MU and MR, we run a stochastic version of Algorithm~\ref{algorithm:full_information_gda}, and constrain the learned distribution to keep sampling a random scenario with probability $0.05$, effectively allowing us to maintain utility estimates across all scenarios.

\paragraph{Results.} The results in Table~\ref{tab:cooking.train} clearly show that MU outperforms all other evaluated methods. Looking at the robustness metrics, the utility formulation has the best worst-case utility ($\wcu$) overall and the best worst-case regret ($\wcr$) on the Melting Pot scenarios. MR also performs better than any other benchmarked method overall, securing the lowest worst-case regrets on both the training set and our own test set. In terms of average performance ($\perf$), MU and MR are consistently the best and second-best, respectively, which is particularly notable on the training set where PBR was expected to perform the best.

One possible explanation for the globally lower performance of the regret approach compared to utility is that the best responses for training scenarios are too approximate. Figure~\ref{fig:cooking.learning_curves} suggests another hypothesis for why MU and MR considerably outperform leaning population best responses and other approaches: their scenario distributions during training have a similar effect to curriculum learning, introducing indirect exploration in behaviours compared to fixed distributions, and a smoother learning curve. In fact, they appear to speed up learning. In contrast, Figure~\ref{fig:cooking:learning_curves.circuit} shows that training against a uniform distribution of scenarios fails to improve in at least one scenario on the Circuit layout, with a worst-case utility close to 0. 


\section{Conclusion}

We explored how to compute robust adaptive policies for Ad Hoc Teamwork. Building on Minimax-Bayes RL, we introduced a method to identify minimax distributions over background populations, which consistently yielded more robust policies compared to training with a uniform distribution. Surprisingly, we also found that training on minimax distributions can significantly accelerate learning.
However, utilising regret as an objective to tune the training distribution proved computationally expensive when best-response utilities are not readily available. In some special instances, we also found that the minimax-Bayes training approach w.r.t.\ utility can prevent policies from learning certain game dynamics.  

Looking forward, we see great potential in extending our approach to a curriculum learning framework based on partner distributions, which could dramatically improve sample efficiency, asymptotic performance, and AHT robustness. 





\section*{Acknowledgements} 
Thomas Kleine Buening is supported by the UKRI Prosperity Partnership Scheme (Project FAIR).




\bibliographystyle{ACM-Reference-Format} 
\bibliography{ref}

\onecolumn
\appendix
\section*{Appendix}

\section{Algorithms}
\label{appendix:algorithms}
We provide an implementation for solving the maximin utility/minimax regret problems when the game is unknown, in Algorithm~\ref{algorithm:unknown_mdp_gda}.

\begin{algorithm}[h]

\caption{Background-Focal SGDA}
\label{algorithm:unknown_mdp_gda}

\begin{algorithmic}[1]
\STATE \textbf{Input\,} 
set of background policies $\mathcal{B}$, batch size $B$,
learning rates ($\eta_\pi$, $\eta_\beta$)

\STATE Simplex projector $\mathcal{P}$
\STATE Initialise randomly the main policy parameters $\theta_0$
\STATE Initialise the belief as the uniform distribution over possible scenarios $\beta_0 = \mathcal{U}(\Sigma(\mathcal{B}))$  

\FOR{$t=0, \dots, N-1$}

    \STATE Sample $B$ scenarios $\sigma_1, \dots, \sigma_B \sim \beta_t$
    \STATE Estimate $\hat U(\pi_{\theta_t}, \sigma_i)$ by deploying $\pi_{\theta_t}$ on $\sigma_i$, for  $i = 1, \dots, B$
    \STATE Compute $\hat U(\pi_{\theta_t}, \beta_t) = \dfrac{1}{B} \sum_{i=1}^{B} \hat U(\pi_{\theta_t}, \sigma_i)$
    
    \STATE Compute $\hat R(\pi_{\theta_t}, \sigma_i) = U^*(\sigma_i) - \hat U(\pi_{\theta_t}, \sigma_i)$ for each  $i = 1, \dots, B$

    \STATE Compute $\hat L(\pi_{\theta_t}, \beta_t) = \dfrac{1}{B} \sum_{i=1}^{B} \hat R(\pi_{\theta_t}, \sigma_i)$
    \IF{solving maximin utility}
        \STATE Update belief $\beta_{t+1} = \mathcal{P}(\beta_t - \eta_{\beta}\nabla_\beta \hat U(\pi_{\theta_t}, \beta_t))$
    \ENDIF
    \IF{solving minimax regret}
    \STATE Update belief $\beta_{t+1} = \mathcal{P}(\beta_t + \eta_{\beta}\nabla_\beta \hat L(\pi_{\theta_t}, \beta_t))$
    \ENDIF
    
    \STATE Update policy parameters $\theta_{t+1} = \theta_t + \eta_{\theta}\nabla_\theta \hat U(\pi_{\theta_t}, \beta_t)$
\ENDFOR

\RETURN $\theta^*, \beta^*$ uniformly at random from $\{(\theta_1, \beta_1), \dots, (\theta_N, \beta_N)\}$
\end{algorithmic}
\end{algorithm}

\section{Ommitted proofs}

\begin{proof}[Proof of Lemma~\ref{lemma:scenario_equivalence}]
Let $\scenario$ be a scenario in $S$. $\scenarioset$ is an $\epsilon$ net for $S$, then there must exist a scenario $\scenario'$ such that $d(\scenario, \scenario')<\epsilon$. 
This implies for the induced transition function that for any $\afocal, s$:
        \begin{align}
            \sum_{s'}\big|P_\scenario(s'| s, \afocal) - P_{\scenario'}(s'| s, \afocal) \big| & \leq \sum_{s'}\sum_{\abackground} P(s'| s, \afocal, \abackground) \big|\prod_i{\bpi}_i^b((\abackground_i \mid h_i)-\prod_i{\bpi'}_i^b(\abackground_i \mid h_i) \big| \\
            &\leq \sum_{\abackground} \sum_{s'} P(s'| s, \afocal, \abackground) \sum_i \big|{\bpi}_i^b(\abackground_i \mid h_i) - {\bpi'}_i^b(\abackground_i \mid h_i) \big|  \\
            &= \sum_{\abackground} \sum_i \big|{\bpi}_i^b((\abackground_i \mid h_i) - {\bpi'}_i^b(\abackground_i \mid h_i) \big|\sum_{s'} P(s'| s, \afocal, \abackground) \\
            &= \sum_{\abackground} \sum_i \big|{\bpi}_i^b((\abackground_i \mid h_i) - {\bpi'}_i^b(\abackground_i \mid h_i) \big| \\
            & < \epsilon.
        \end{align}
Similarly with the induced reward function, we also have for any $\afocal, s,  i$: 
        \begin{align}
            \label{eq:reward_similarity}
            \big|\rewardfunc_\scenario(s, \afocal, i) - \rewardfunc_{\scenario'}(s, \afocal, i) \big| & \leq \sum_{\afocal} |\rewardfunc(s, \afocal, \abackground, i)| \sum_j \big| {\bpi}_j^b((\abackground_j \mid h_j)-{\bpi'}_j^b((\abackground_j \mid h_j) \big| \\ 
            & \leq \rmax \sum_{\abackground} \sum_j \big| {\bpi}_j^b((\abackground_j \mid h_j)-{\bpi'}_j^b((\abackground_j \mid h_j) \big| \\
            & < \epsilon\rmax,
        \end{align}
        where $\rmax$ denotes the maximal absolute step reward.

        We use the notation $U_t(\pi, \scenario, s)$ to denote the utility of a policy $\pi$ when $s_t=s$, with $T-t$ timesteps remaining until the episode terminates. We will prove by induction that for any $t \in [1\dots T]$ and any $s$:
        \begin{equation}
             |U_t(\pi, \scenario, s)-U_t(\pi, \scenario', s)| < \epsilon\rmax (T-t+1)(T-t)/2.
        \end{equation}
        At $t = T$, we have:
        \begin{align}
         & |U_T(\pi, \sigma, s) - U_T(\pi, \sigma', s)| \\
         &= |\sum_{\afocal} \pi(\afocal, h_T) \dfrac{1}{c}\sum_i^c\big[\rewardfunc_\scenario(s, \afocal, i) - \rewardfunc_{\scenario'}(s, \afocal, i)\big]| \\
         & < \epsilon\rmax.
        \end{align}
        For some $t \in [1\dots T]$, assume that the inductive hypothesis holds at $t+1$: $|U_{t+1}(\pi, \scenario, s)-U_{t+1}(\pi, \scenario', s)| \leq \epsilon\rmax(T-t)(T-t-1)$. We have:
        \begin{align}
         |U_t(\pi, \sigma, s) - U_t(\pi, \sigma', s)| & \leq \sum_{\afocal} \pi(\afocal, h_t)
         (\dfrac{1}{c}\sum_i^c \big| \rewardfunc_\sigma(s, \afocal, i) - \rewardfunc_{\sigma'}(s, \afocal, i)\big|) \\
         & + \sum_{s'\in \states} \big|P_\scenario(s'|\afocal, s) U_{t+1}(\pi, \scenario, s') - P_{\scenario'}(s'|\afocal, s) U_{t+1}(\pi, \scenario', s') \big| \\
         & < \sum_{\afocal} \pi(\afocal, h_t) \Big[\epsilon\rmax \\
         &+ \sum_{s'\in \states} | P_\scenario(s'|\afocal, s) - P_{\scenario'}(s'|\afocal, s) | |U_{t+1}(\pi, \scenario, s')| \\
         &+ \sum_{s'\in \states}P_{\scenario'}(s'|\afocal, s)|U_{t+1}(\pi, \scenario, s') - U_{t+1}(\pi, \scenario', s')| \Big] \\
         & < \sum_{\afocal} \pi(\afocal, h_t) \epsilon\rmax\Big[1 + (T-t) + \dfrac{1}{2}(T-t)(T-t-1) \Big] \\
         & = \dfrac{1}{2}\epsilon\rmax(T-t+1)(T-t).
        \end{align}
        This proves by induction that for any $t \in [1\dots T]$ and $s$, $|U_t(\pi, \scenario, s)-U_t(\pi, \scenario', s)| < \epsilon\rmax (T-t+1)(T-t)/2$.
        The stated result is recovered with $t=1$.
    \end{proof}

\begin{proof}[Proof of Lemma~\ref{lemma:scenario_equivalence_regret}]
    Let $\scenario$ be a scenario from $S$. $\scenarioset$ is an $\epsilon$-net for $S$, therefore there is a scenario $\scenario'\in\scenarioset$ such that $d(\scenario, \scenario') < \epsilon$.
    We begin by writing that for any policy $\pi\in\policies$, $U(\pi, \scenario) \leq |U(\pi, \scenario) - U(\pi, \scenario')| +  U(\pi, \scenario')$. Taking the maximum over policies on both sides, we get
    \begin{align}
        U^*(\scenario) &\leq \max_\pi |U(\pi, \scenario) - U(\pi, \scenario')| +  U(\pi, \scenario') \\
        & \leq \max_\pi |U(\pi, \scenario) - U(\pi, \scenario')| + U^*(\scenario')\\
        & \leq \dfrac{\epsilon T^2\rmax}{2} + U^*(\scenario'),
    \end{align}
    where the last inequality is obtained through Lemma~\ref{lemma:scenario_equivalence}.
    Hence moving $U^*(\scenario')$ to the left hand side, we get $U^*(\scenario)- U^*(\scenario') \leq \epsilon T^2\rmax/2 $. As the above logic also holds by swapping the roles of $\scenario$ and $\scenario'$, we can further note that $|U^*(\scenario)- U^*(\scenario')| \leq \epsilon T^2\rmax/2 $.
    Finally, by the definition of regret, it holds for any policy $\pi\in\policies$ that
    \begin{align}
        |R(\pi, \scenario)-R(\pi, \scenario')| &\leq |U^*(\scenario)- U^*(\scenario')| + |U(\pi, \scenario) - U(\pi, \scenario')| \\
        & < \epsilon T^2\rmax,
    \end{align}
    which concludes the proof.
\end{proof}

\begin{proof}[Proof of Lemma~\ref{lemma:wcu.guarantees}]
     Let $\scenario_\text{wc}(\scenarioset)$ and $\scenario_\text{wc}(S)$ be the worst-case scenarios for $\pi^*_U$ on $\scenarioset$ and $S$, respectively: $\wcu(\pi^*_U, \scenarioset) = U(\pi^*_U, \scenario_\text{wc}(\scenarioset))$ and $\wcu(\pi^*_U, S) = U(\pi^*_U, \scenario_\text{wc}(S))$. Because $\pi^*_U$ is the solution of the maximin utility problem \eqref{eq:mbmarl.maximin} on $\scenarioset$, we can also write that $U(\pi^*_U, \scenario_\text{wc}(\scenarioset)) = \max_\pi \wcu(\pi, \scenarioset)$. Now, because $\scenarioset$ is an $\epsilon$-net for $S$, we have two cases: 
    \begin{enumerate}
        \item $d(\scenario_\text{wc}(\scenarioset), \scenario_\text{wc}(S)) < \epsilon$. Lemma~\ref{lemma:scenario_equivalence} can be applied, obtaining $U(\pi^*_U, \scenario_\text{wc}(S)) > U(\pi^*_U, \scenario_\text{wc}(\scenarioset)) - \epsilon T^2 \rmax/2$.
        
        \item $d(\scenario_\text{wc}(\scenarioset), \scenario_\text{wc}(S)) \geq \epsilon$, meaning that there is another scenario $\scenario_\epsilon \in \scenarioset$ such that  $d(\scenario_\epsilon, \scenario_\text{wc}(S)) < \epsilon$.
        Since we have $U(\pi^*_U, \scenario_\text{wc}(S)) > U(\pi^*_U,\scenario_\epsilon) - \epsilon T^2 \rmax/2$ through Lemma~\ref{lemma:scenario_equivalence}, and because of the definition of $\scenario_\text{wc}(\scenarioset)$ we have $U(\pi^*_U, \scenario_\epsilon)\geq U(\pi^*_U, \scenario_\text{wc}(\scenarioset))$, the desired result is obtained by transitivity:
        \begin{align}
            \wcu(\pi^*_U, S) &= U(\pi^*_U, \scenario_\text{wc}(S))\\
            &> U(\pi^*_U,\scenario_\epsilon) - \dfrac{\epsilon T^2\rmax}{2}\\
            & \geq U(\pi^*_U, \scenario_\text{wc}(\scenarioset)) - \dfrac{\epsilon T^2\rmax}{2}\\
            & = \max_\pi \wcu(\pi, \scenarioset) - \dfrac{\epsilon T^2\rmax}{2}.
        \end{align}
    \end{enumerate}
\end{proof}

\begin{proof}[Proof of Lemma~\ref{lemma:wcr.guarantees}]
         Let $\scenario_\text{wc}(\scenarioset)$ and $\scenario_\text{wc}(S)$ be the worst-case scenarios for $\pi^*_R$ on $\scenarioset$ and $S$, respectively: $\wcr(\pi^*_R, \scenarioset) = R(\pi^*_R, \scenario_\text{wc}(\scenarioset))$ and $\wcr(\pi^*_R, S) = R(\pi^*_R, \scenario_\text{wc}(S))$. Because $\pi^*_R$ is the solution of the minimax regret problem \eqref{eq:mbmarl.minimax} on $\scenarioset$, we can also write that $R(\pi^*_R, \scenario_\text{wc}(\scenarioset)) = \min_\pi \wcr(\pi, \scenarioset)$. Now, because $\scenarioset$ is an $\epsilon$-net for $S$, we have two cases: 
    \begin{enumerate}
        \item $d(\scenario_\text{wc}(\scenarioset), \scenario_\text{wc}(S)) < \epsilon$. Lemma~\ref{lemma:scenario_equivalence_regret} can be applied, obtaining $R(\pi^*_R, \scenario_\text{wc}(S)) > R(\pi^*_R, \scenario_\text{wc}(\scenarioset)) + \epsilon T^2 \rmax$.
        
        \item $d(\scenario_\text{wc}(\scenarioset), \scenario_\text{wc}(S)) \geq \epsilon$, meaning that there is another scenario $\scenario_\epsilon \in \scenarioset$ such that  $d(\scenario_\epsilon, \scenario_\text{wc}(S)) < \epsilon$.
        Since we have $R(\pi^*_R, \scenario_\text{wc}(S)) > R(\pi^*_R,\scenario_\epsilon) + \epsilon T^2 \rmax$ through Lemma~\ref{lemma:scenario_equivalence_regret}, and because of the definition of $\scenario_\text{wc}(\scenarioset)$ we have $R(\pi^*_R, \scenario_\epsilon)\leq R(\pi^*_R, \scenario_\text{wc}(\scenarioset))$, the desired result is obtained by transitivity:
        \begin{align}
            \wcr(\pi^*_R, S) &= R(\pi^*_R, \scenario_\text{wc}(S))\\
            &< R(\pi^*_R,\scenario_\epsilon) + \epsilon T^2\rmax\\
            & \leq R(\pi^*_R, \scenario_\text{wc}(\scenarioset)) + \epsilon T^2\rmax\\
            & = \min_\pi \wcr(\pi, \scenarioset) + \epsilon T^2\rmax.
        \end{align}
    \end{enumerate}
\end{proof}

\section{Additional experimental results}

\paragraph{Repeated Prisoner's Dilemma}
In this section of the experiments, we utilised nine specific ad-hoc policies to create the set of training partners $\poptrain$. The set was constructed as follows:
\begin{itemize}
    \item A pure cooperative policy: always chooses to cooperate.
    \item A pure defecting policy: always chooses to defect.
    \item Two tit-for-tat policies: \begin{itemize}
        \item One that begins with cooperation and then mimics the opponent's last action in subsequent rounds.
        \item Another that starts with defection and also replicates the opponent's last action thereafter.
    \end{itemize}
    \item Two tat-for-tit policies: \begin{itemize}
        \item One that starts by cooperating and subsequently takes the opposite action of the opponent's last move.
        \item Another that begins by defecting and then does the same.
    \end{itemize}
    \item A cooperate-until-defected policy: cooperates until it encounters a defection from its opponent.
    \item A defect-until-cooperated policy: defects until it observes cooperation from its opponent.
    \item A fully random policy: selects actions randomly, without any strategy.
\end{itemize}
We provide a plot of the performance/robustness metrics of the learned policies in function of how novel the test scenarios are in Figure~\ref{fig:prisoners.epsilon_plot}. As test scenarios become more distinct from the training ones, PBR experiences the largest performance drop. In comparison, MR's worst-case regret shows only a slight increase, and MU maintains consistent worst-case utility. There reason for MU's policy constant worst-case utility/regret is that its worst-case scenario in both cases is the universalisation scenario (where it plays against itself) for all values of $\epsilon$. No matter how different are the test policies from the training policies, one policy's utility/regret on the universalisation scenario won't change (there is no background player in this scenario).
\begin{figure}
    \centering
    \begin{subfigure}[t]{0.51\linewidth}
    \includegraphics[width=\linewidth]{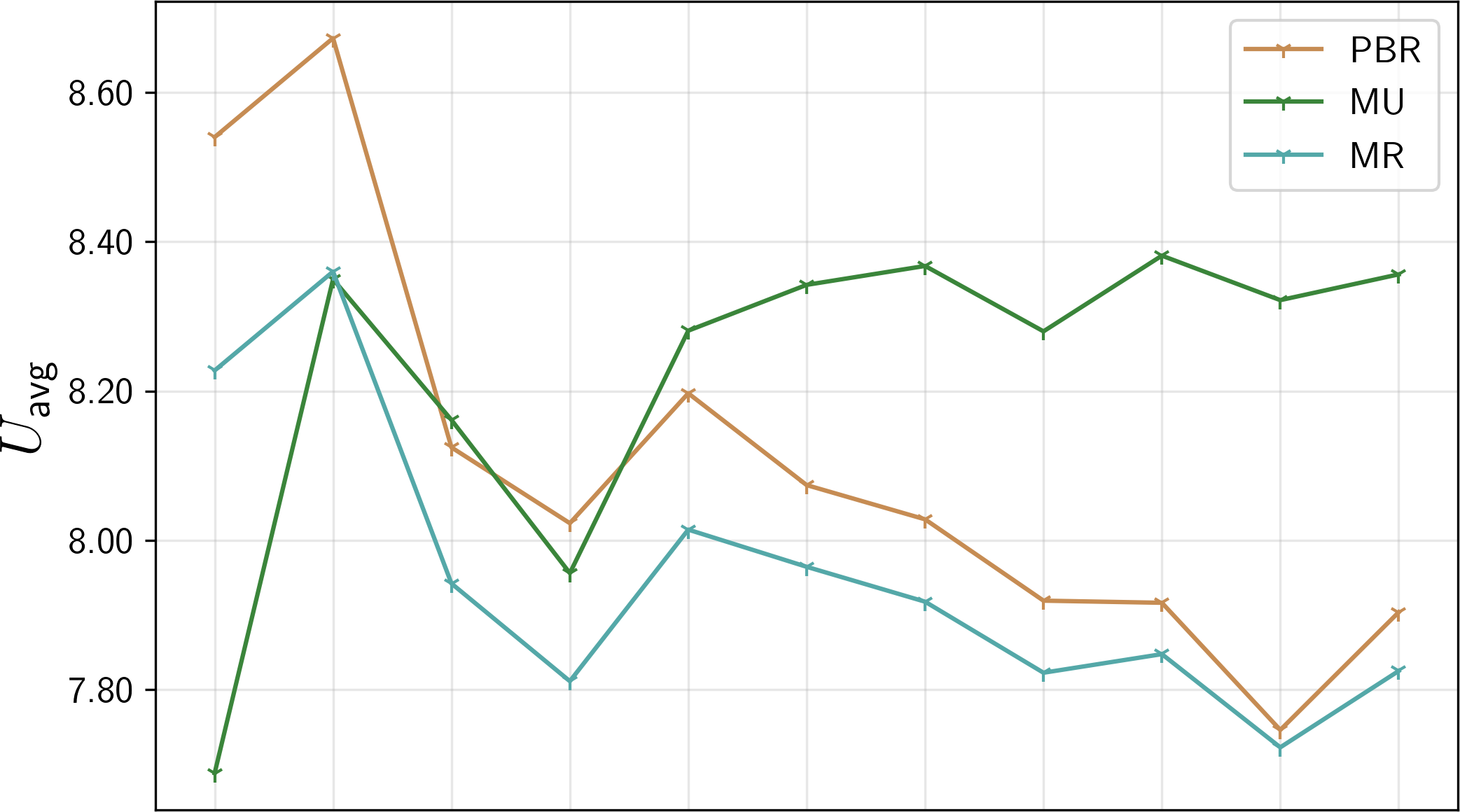}
    \end{subfigure}
    \begin{subfigure}[t]{0.51\linewidth}
    \includegraphics[width=\linewidth]{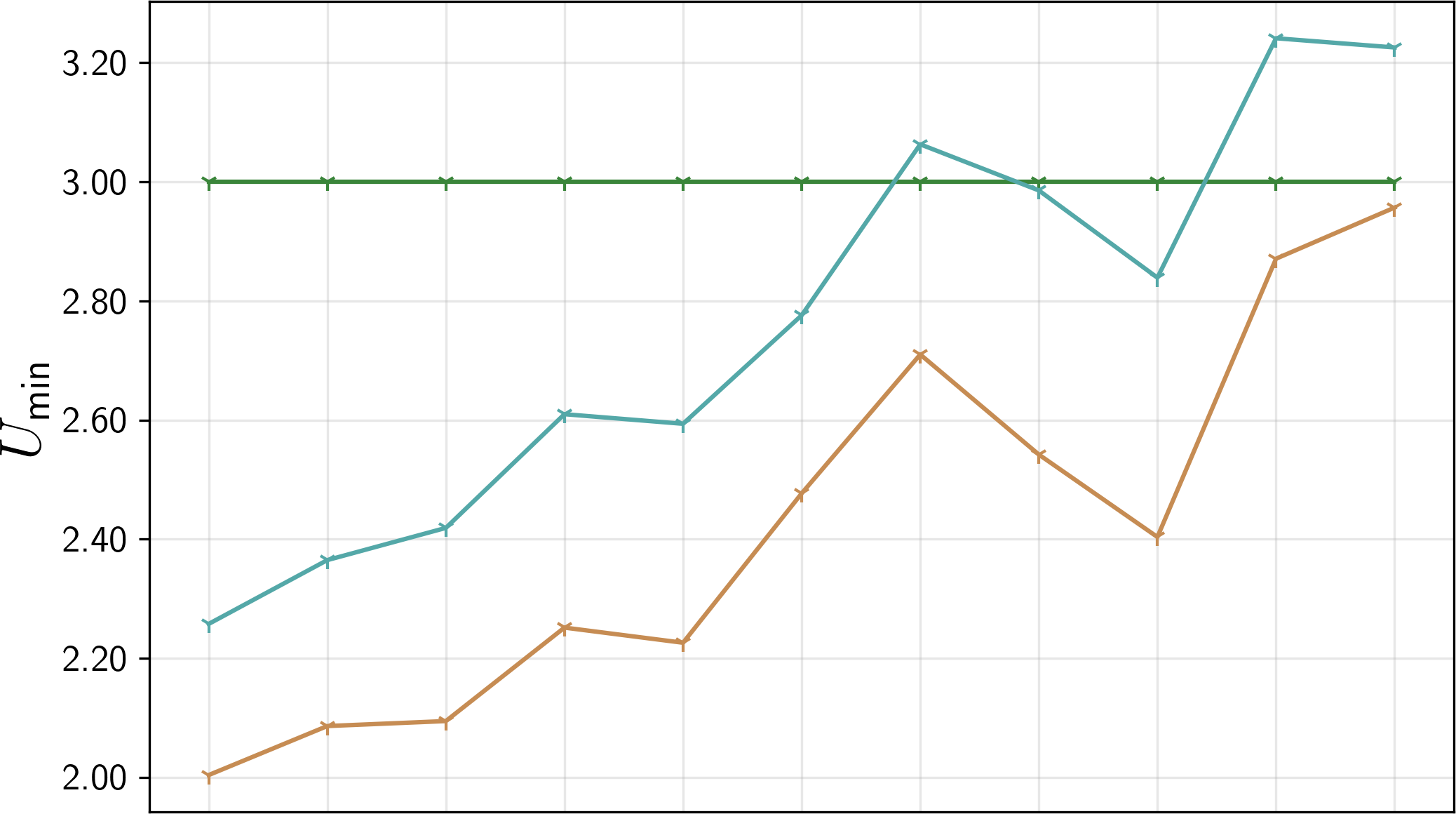}
    \end{subfigure}
    \begin{subfigure}[t]{0.51\linewidth}
    \includegraphics[width=\linewidth]{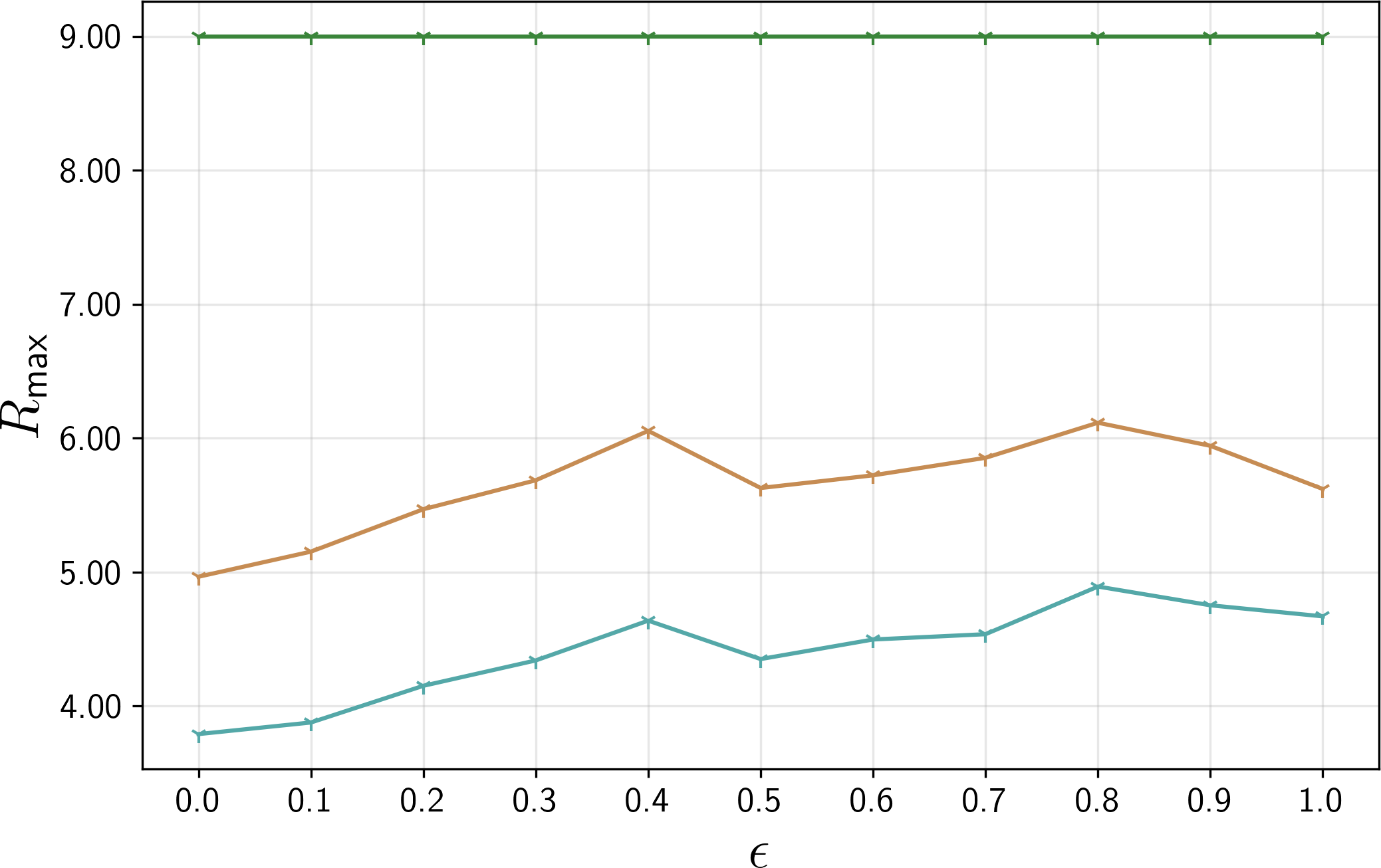}
    \end{subfigure}
    \caption{Average utility, worst-case utility and worst-case regret on the test set, in function of the $\epsilon$-net width (1 run).}
    \label{fig:prisoners.epsilon_plot}
\end{figure}

\paragraph{Robust AHT on Collaborative Cooking}

We plotted the learning curves of distributions when trying to solve the maximin utility/minimax regret problems with Algorithm~\ref{algorithm:unknown_mdp_gda} on Collaborative Cooking, in Figure~\ref{fig:cooking.learning_curves_prior}. This displays the clear differences between distributions learned by MU and MR on both kitchen layouts. The low variance between runs also highlights the consistency of the stochastic GDA Algorithm~\ref{algorithm:unknown_mdp_gda}. Note that the plots are highly smoothed, as the scenario distribution fluctuate a lot more than the policy.
\begin{figure}
    \centering
    \begin{subfigure}[t]{0.48\linewidth}
    \includegraphics[width=\linewidth]{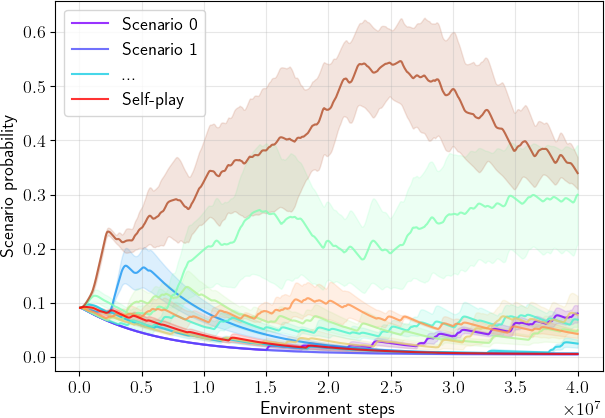}
    \caption{MU, Circuit layout.}
    \end{subfigure}
    \begin{subfigure}[t]{0.48\linewidth}
    \includegraphics[width=\linewidth]{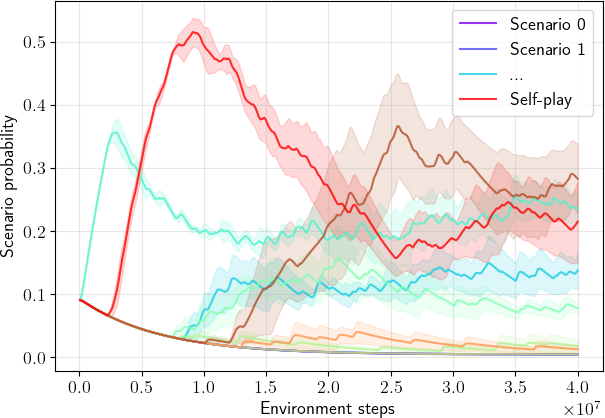}
    \caption{MR, Circuit layout.}
    \end{subfigure}
        \begin{subfigure}[t]{0.48\linewidth}
    \includegraphics[width=\linewidth]{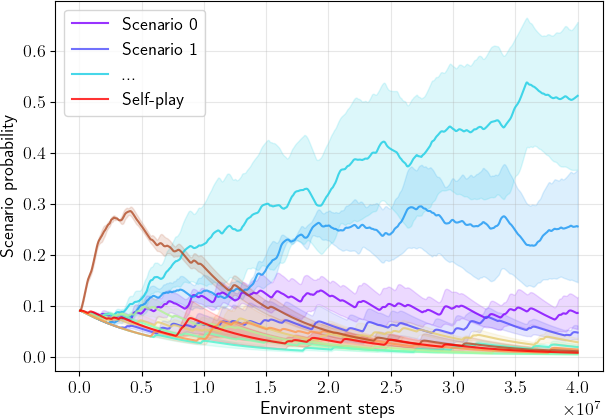}
    \caption{MU, Cramped layout.}
    \end{subfigure}
    \begin{subfigure}[t]{0.48\linewidth}
    \includegraphics[width=\linewidth]{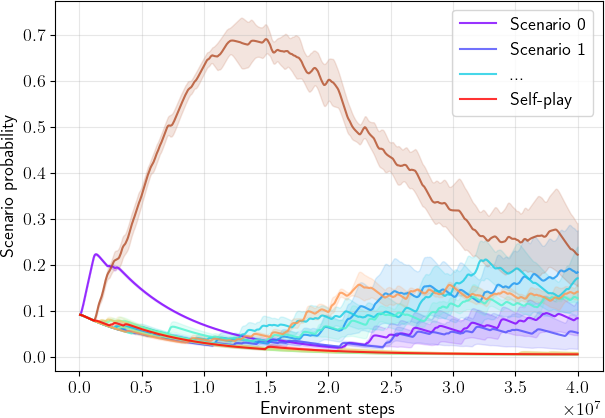}
    \caption{MR, Cramped layout.}
    \end{subfigure}
    \caption{Smoothed learning curves of the distribution over the training scenarios (averaged over 3 runs).}
    \label{fig:cooking.learning_curves_prior}
\end{figure}

\section{Additional experimental details}

\paragraph{Robust AHT on Collaborative Cooking}
To facilitate the training of our policies in Collaborative Cooking, we used a shaping pseudo-reward of $1$ when tomatoes were placed in the cooking pot. For the background policies, we further altered the reward function to restrict delivery rewards to the player that delivered. Combining this new reward function with varying levels of prosociality and risk-aversion helped the background policies adopt diversified ways to solve the game.

The architecture for the agents consisted of a convolutional network with two layers, having 16 and 32 output channels, kernel shapes of 8 and 4, and strides of 8 and 1, respectively. The output of the convNet was concatenated with the previous action taken before being passed into a dense layer of size 256 and an LSTM with 256 units. Policy and baseline (for the critic) were produced by linear layers connected to the output of the LSTM. Notably, to increase stability, we made the critic output a separate value for each training scenario.

We chose PPO to train our policies, using the Adam optimiser with a learning rate of $2 \times 10^{-4}$, a discount factor of $0.99$, a GAE lambda of $0.95$, a KL coefficient of $1.0$ with a KL target of $0.01$, and a PPO clipping parameter of $0.3$. Gradients were clipped at 4.0. We did not employ entropy regularisation. PPO was set to run 2 epochs per batch, each containing $64000$ samples, with minibatches of $1000$ samples each. Finally, the unroll length for the LSTM was set at $20$.

For the prior, we used a learning rate of $0.4$. 

\paragraph{Implementation}
The code used for all of our experiments is available at \hyperlink{https://github.com/AkiEl9/minimax_bayes_aht}{\url{https://github.com/AkiEl9/minimax_bayes_aht}}.

\end{document}